\documentclass[11pt,twocolumn]{article}

\usepackage[margin=1in]{geometry}
\usepackage{amsmath,amssymb,amsthm,mathtools}
\usepackage{booktabs}
\usepackage{multirow}
\usepackage{graphicx}
\usepackage{hyperref}
\usepackage{enumitem}
\usepackage{bm}
\usepackage[expansion=false]{microtype}
\usepackage{natbib}             
\usepackage{array}
\usepackage{xspace}
\usepackage{tikz}
\usetikzlibrary{arrows.meta, positioning, shapes.geometric}
\hypersetup{
    colorlinks=true,
    linkcolor=blue,
    citecolor=blue,
    urlcolor=blue
}
\newtheorem{theorem}{Theorem}[section]
\newtheorem{proposition}[theorem]{Proposition}

\theoremstyle{definition}

\theoremstyle{remark}

\newcommand{\E}{\mathbb{E}}

\newcommand{\KL}{\mathrm{KL}}
\newcommand{\JS}{\mathrm{JS}}

\newcommand{\Deff}{D_{\mathrm{eff}}}
\newcommand{\Sexc}{S_{\mathrm{exc}}}
\newcommand{\deltap}{\Delta p_1}
\newcommand{\deltastar}{\Delta p_1^*}
\newcommand{\lstruct}{\lambda_{\mathrm{struct}}}
\newcommand{\qbox}[1]{%
  \begin{center}
  \fbox{\parbox{0.82\columnwidth}{\centering #1}}
  \end{center}
}
\newcommand{\llocal}{\lambda_{\mathrm{local}}}
\numberwithin{equation}{section}

\makeatletter
\renewcommand{\@maketitle}{%
  \newpage
  \null
  \vskip 0.5em
  \begin{center}%
  \let \footnote \thanks
    {\LARGE \@title \par}%
    \vskip 1em%
    {\large
      \lineskip .5em%
      \begin{tabular}[t]{c}%
        \@author
      \end{tabular}\par}%
  \end{center}%
  \par
  \vskip 0.5em}
\makeatother

\title{
Paragraph Boundaries Are Not White Space:
\\[0.4em]
\large
Compression Depth as the Signature of Hierarchical Structure
}

\author{Shuyang Xiang}
\date{}

\begin{document}

\maketitle

\begin{abstract}
Standard positional encodings treat position as a one-dimensional
reading-order coordinate, but reading order alone does not determine
hierarchical textual structure. We use a hierarchical rotary positional
encoding (hRoPE) that represents paragraph, sentence, and token indices
as separate channels, hold the token sequence fixed, intervene on the
paragraph coordinate $p_1$, and measure cross-paragraph attention with a
token-distance-exact estimator. Attention is compressed in every corpus,
but compression alone is not diagnostic of true structure: an
architecturally identical channel with density-matched random labels is
compressed too. What distinguishes real structure is the depth of
compression: the real-versus-random depth gap is resolvable in two of
three corpora and not in the third, and real-structure depth
varies more strongly across corpora than the control's. Comparing eight corpus-only quantities across three
constructs (lexical persistence, paragraph length, embedding-based
coherence), none fully reproduces the cross-corpus ordering of depth,
though embedding-based coherence comes closest. At the paragraph level
that relation transfers as a common slope, but with the opposite
sign to the corpus-level ranking: controlling for length, diversity, and
position, paragraphs with more similar neighbors compress less deeply,
with no detectable slope difference between any pair of corpora, while
length, lexical-diversity, and position effects remain corpus-specific. Compression depth,
not its location, is the reproducible signature of genuine paragraph
structure in our setting.
\end{abstract}

\section{Introduction}
\label{sec:intro}

Start by considering two sequences containing exactly the same sentences:
\begin{quote}
Sentence A. Sentence B. Sentence C.\\[6pt]
\hfill\emph{versus}\\[6pt]
Sentence A.\\[2pt]
Sentence B.\\[2pt]
Sentence C.
\end{quote}
Apart from an explicit paragraph separator---a blank-line (newline) token in the second rendering---inserted before Sentences B and C, the two sequences contain the same tokens in the same order, at the same reading-order distances. Yet a human reader does not process them the same way---a paragraph boundary can induce a pause, mark a change of local context, or alter how earlier material relates to what follows. Hence a simple question:

\qbox{Is attention compressed near a paragraph boundary, and if so, is the depth of that compression greater under genuine paragraph structure than under a matched random control?}

More precisely: does a paragraph boundary change how the model relates tokens across it, beyond what their reading-order separation predicts, and is that compression a causal effect of hierarchical position rather than of information density? This is narrower than asking whether language models ``understand'' paragraphs: we ask whether hierarchical structure changes the effective interaction distance between tokens, and by how much.

Standard positional encodings represent position through a one-dimensional reading-order coordinate: the relationship between tokens at $i$ and $j$ is derived from $i$ and $j$, or from $i-j$. But text is nested, not linear ($\text{document} \supset \text{paragraph} \supset \text{sentence} \supset \text{token}$): two tokens with the same reading-order separation can be in the same sentence, in different sentences of one paragraph, or in different paragraphs---cases reading order alone cannot distinguish.

Changing only the paragraph label of ``You'' in ``This is what I love. You are the reason''---with no token changed---produces a measurable drop in attention across the boundary, and the effect vanishes in control models without a paragraph-position channel (Figure~\ref{fig:love-demo}, Appendix~\ref{app:demo}). This leads to a first, conceptual question: \textbf{Is reading-order position sufficient to represent hierarchical textual position?} We show that it is not: an $n$-token sequence admits $2^{n-1}$ paragraph segmentations, so one reading order corresponds to exponentially many hierarchical structures---reading order is exact for the linear sequence but incomplete for the hierarchy. This does not itself imply that a Transformer uses hierarchical position, only that reading order may not be a sufficient statistic for arbitrary paragraph structure. We therefore construct a hierarchical positional representation with separate paragraph, sentence, and token coordinates, enabling counterfactual interventions on paragraph position with the token sequence held fixed. It does compress attention---but not exclusively: an architecturally identical channel with density-matched, per-step-resampled random labels is compressed too, so compression alone is not evidence of true structure. What distinguishes real paragraph structure is \textbf{depth}: real structure compresses more deeply than the control in most corpora, and this depth varies more strongly across corpora than the control's. Whether corpus properties, rather than the model, can account for that corpus-specific depth is the question we turn to last.

\begin{figure*}[t]
\centering
\includegraphics[width=\textwidth]{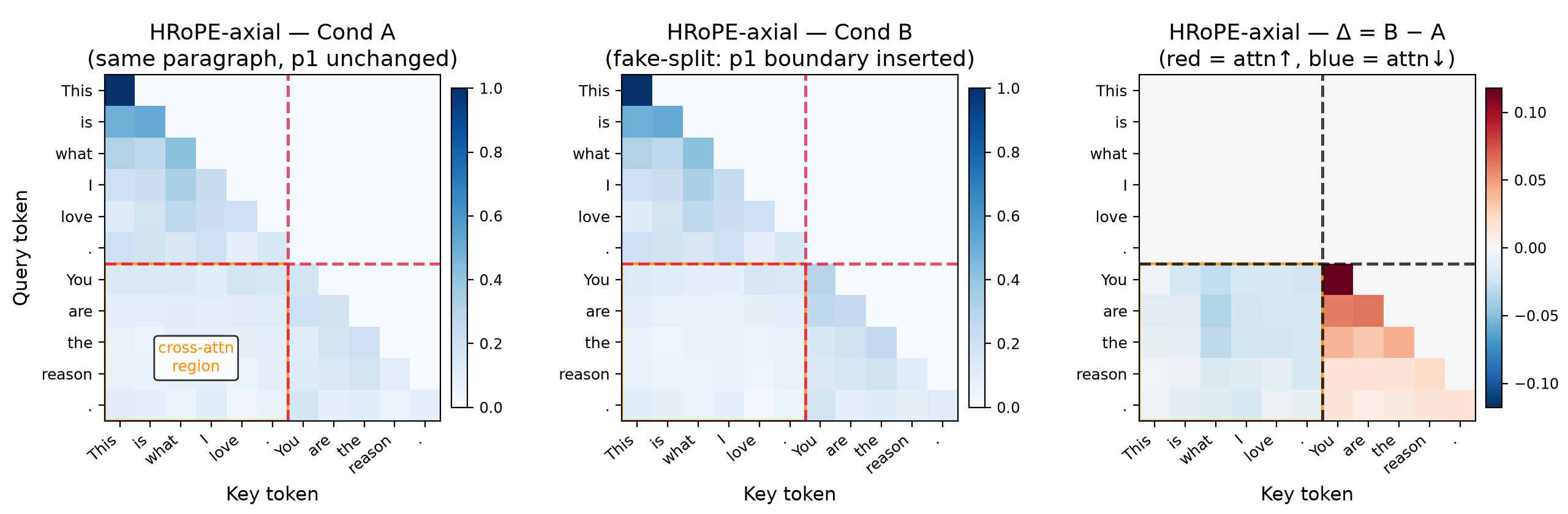}
\caption{Changing only the paragraph coordinate $p_1$ of ``You'', with the
token sequence identical, changes attention: \textbf{hrope\_axial} maps for
the original (Cond A) and fake-split (Cond B) conditions and their
difference. Controls and $\Deff$ values: Figure~\ref{fig:love-demo-full}.}
\label{fig:love-demo}
\end{figure*}

This paper proceeds in three layers. Section~\ref{sec:layer1} isolates the causal role of the paragraph coordinate, ruling out token-content artifacts; Section~\ref{sec:layer2} establishes the central claim, that compression is deeper under real structure than under a matched random control and is corpus-dependent; Section~\ref{sec:layer3} compares corpus-only candidate quantities with the observed ordering of depth; Section~\ref{sec:discussion} states the limitations, and Section~\ref{sec:conclusion} concludes.

\paragraph{Contributions} Our contributions are threefold:
\begin{itemize}
\item \textbf{Precondition}: reading order is not a sufficient coordinate for hierarchical structure, and intervening on $p_1$ with tokens held fixed causally changes attention in all three corpora (Section~\ref{sec:layer1}).
\item \textbf{Characterization}: attention is compressed near paragraph boundaries in every corpus, but compression alone is not diagnostic---depth, not location, separates real from random structure (Section~\ref{sec:layer2}).
\item \textbf{Comparison with corpus-only candidates}: no corpus-only quantity reproduces the cross-corpus ordering of depth, though embedding-based coherence comes closest, and at the paragraph level the coherence slope---opposite in sign to the corpus-level ranking---is common across corpora (Section~\ref{sec:layer3}).
\end{itemize}

\section{Related Work}
\label{sec:related}

\subsection{Positional encodings beyond reading order}

Conventional positional encodings tell Transformers about token order: RoPE \citep{su2021roformer} encodes relative displacement through rotations, and ALiBi \citep{press2022train} adds position-dependent attention biases; both operate on a linear, scalar position. A separate line encodes structure other than reading order: \citet{wang2019self} encode a token's position within a sentence's dependency tree rather than its linear index, while LieRE \citep{ostmeier2025liere} generalizes RoPE from fixed 2D rotations to learned dense rotations parameterized by Lie algebra generators. Our multi-axis hRoPE also encodes more than a scalar position, but assigns each hierarchical coordinate an independent, commuting block of channels, so that we can isolate and intervene on one coordinate ($p_1$) at a time rather than maximize capacity. Unlike all of the above, we ask whether a Transformer responds causally to an explicit higher-level coordinate, and whether that response has a coherent shape.

\subsection{Hierarchical architectures and long-document modeling}
Several architectures introduce hierarchical representations for long
documents: HIBERT \citep{zhang2019hibert} and related hierarchical encoders build separate sentence-document Transformer stacks rather than modifying the positional encoding. PermGen \citep{yu2021sentence} combines a sentence-level (global) index with a within-sentence (local) index, motivated by sentence-order permutation in generation rather than by testing whether the coordinate causally affects attention. HDT \citep{he2024hdt} assigns each token a per-level linear index across a token/sentence/section hierarchy and sparsifies attention along it, while Longformer \citep{beltagy2020longformer} sparsifies attention without an explicit hierarchical coordinate. Our construction is close to PermGen and HDT, but we do not use the coordinate to improve generation or efficiency; we intervene on it with all tokens fixed, and ask whether the resulting attention response has a coherent, corpus-predictable shape.

\subsection{Geometric framings and causal intervention methodology}

Independent of positional encoding design, recent work asks whether Transformer computation admits a geometric or physical framework. \citet{cirrincione2026geometry} characterizes what an information-optimal positional encoding must satisfy, showing that the geometry induced by positional token distributions is generally curved and cannot be exactly Euclidean---complementary to our model-derived, causally verified characterization. \citet{disipio2025curved} propose a general-relativistic analogy in which query-key interactions induce an effective metric on representation space; they study the evolution of token representations across layers, whereas we study how attention weight depends on an explicit hierarchical coordinate, via direct counterfactual intervention. Our fake-$p_1$ intervention is in the tradition of causal mediation analysis and activation patching \citep{vig2020causal}, which manipulate a model's internal state or input to estimate a component's causal contribution; ours instead swaps only an explicit positional coordinate with the token sequence entirely fixed. Unlike comparisons based on perplexity or downstream performance, we falsify the positional variable directly and ask whether the resulting sensitivity is anticipated by corpus statistics alone.

\section{Experimental Framework}
\label{sec:framework}

This section defines the hierarchical coordinate system, the hRoPE architecture, the model variants (four primary ones plus the \textbf{period\_axial} mirror control, Section~\ref{sec:hrope}), and the causal-intervention and bootstrap methodology used throughout.

\subsection{Hierarchical coordinate system and hRoPE}
\label{sec:hrope}

We introduce a hierarchical coordinate system $(p_1,p_2,p_3)$: $p_1$ is the paragraph index, $p_2$ the sentence index within the paragraph, and $p_3$ the token index within the sentence, so $(3,4,5)$ is the fifth token of the fourth sentence of the third paragraph. Neither reading order nor the coordinate can be recovered from the other alone, so rather than compressing them into a single scalar, hRoPE assigns each level its own rotary channel group:
\[
R_{\rm hier}(p_1,p_2,p_3) = R_1(p_1) \oplus R_2(p_2) \oplus R_3(p_3).
\]
where each $R_k$ is a rotary transformation acting on its own channel subspace. This axial construction allows us to manipulate the paragraph coordinate $p_1$ independently of token identities and lower-level coordinates, which is essential for the counterfactual interventions below.

Standard RoPE~\citep{su2021roformer} encodes a scalar position $m$ by rotating each channel pair through angle $m\theta_k$, so the attention score depends only on relative displacement; in hRoPE this property holds independently within each channel, and $p_1$ is never combined with $p_2$ or $p_3$ into a shared value. Each group's frequency base is calibrated independently from its corpus distribution, with no learned parameters (Appendix~\ref{app:training}).

We train four positional variants, matched in parameter count, depth, training steps, and seeds: \textbf{flat} (standard RoPE, no hierarchical coordinate); \textbf{sent\_axial} (groups for $p_2,p_3$ only); \textbf{hrope\_axial} (full three-group construction); and \textbf{rand\_axial} (identical, but $p_1$ is a resampled label $\rho(i)$ matched to boundary density rather than the true position), which separates having an additional coordinate from that coordinate carrying true positional information. A fifth model, \textbf{period\_axial}, is a mirror control whose $p_1$ is the mechanical grid $\lfloor t/L\rfloor$; it enters only the loss-based substitution check of Section~\ref{sec:p1substitution}, and appears neither in any depth ($U^*$) comparison nor in the validation-loss comparison of Appendix~\ref{app:valcost} (Appendix~\ref{app:training}).

\subsection{Counterfactual interventions and distance-residualized attention}

For each document we compare the original hierarchical assignment with two counterfactuals, \emph{fake-merge} and \emph{fake-split}. In \emph{fake-merge}, two genuine paragraphs receive the same paragraph coordinate; in \emph{fake-split}, an artificial boundary is introduced within a paragraph. In both cases the tokens and reading order are unchanged and only $p_1$ is manipulated (Figure~\ref{fig:interventions}, Appendix~\ref{app:interventions}).

To measure attention beyond reading-order distance, we average over layers and heads, $\bar A(i,j) = \frac{1}{LH}\sum_{l=1}^{L}\sum_{h=1}^{H} A_{l,h}(i,j)$.
After this averaging, the leading dependence of the RoPE attention score on token distance is exponential decay; taking logs makes this leading trend linear, which we fit and then subtract, so that the residual
\begin{align}
\log \bar A(i,j) &= \hat\beta_0 + \hat\beta_1 d(i,j) + \epsilon_{ij},\label{eq:fit}\\
\Deff(i,j) &= \log\bar A(i,j) - \hat\beta_0 - \hat\beta_1 d(i,j),\label{eq:deff}
\end{align}
captures the part of the attention pattern not explained by token distance alone: the fit removes the dominant, monotone, distance-driven component. This is only a first-order removal, so any residual, nonlinear dependence on $d$ surviving the linear fit is a known limitation (Appendix~\ref{app:pooling}).

\subsection{Normalized boundary effect and bootstrap protocol}
\label{beta-protocol}
Let $\mathcal{M}$ denote a manipulated condition. With two unweighted anchors (defined in Appendix~\ref{app:stat_methods}) we define
\begin{equation}
\beta = \frac{\overline{\Deff(\mathcal{M})} - \mathrm{anchor}_{\rm within}}{\mathrm{anchor}_{\rm para} - \mathrm{anchor}_{\rm within}},\label{beta}
\end{equation}
which measures the manipulated condition's level relative to the two natural baselines: $\beta=0$ is the within-paragraph baseline and $\beta=1$ the average real paragraph-boundary effect. No distance weighting is applied here, since the $d$-weighted pooling of Section~\ref{sec:exactd} is a separate, later step (Appendix~\ref{app:stat_methods}).
 
The two interventions start from opposite baselines: fake-merge acts on pairs that spanned a genuine boundary, so its unmanipulated level is the real-boundary level ($\beta=1$), while fake-split acts on within-paragraph sentence-boundary pairs, whose unmanipulated level $\beta_A^{\rm sent}$ is small and positive for WikiText-2 and OpenWebText ($\approx+0.2$ each; $\beta=0$ is the idealized limit) but \emph{not} small for Code ($\approx-1.85$). With $\beta_{\rm observed}$ the value of $\beta$ under the manipulated condition, the intervention effect is
\begin{equation}\label{eq:delta}
\Delta = \beta_{\rm observed} - \beta_{\rm unmanipulated},
\end{equation}
where $\beta_{\rm unmanipulated}=1$ for fake-merge and $\beta_A^{\rm sent}$ for fake-split; normalizing $\Delta_C$ against each corpus's own $\beta_A^{\rm sent}$ removes the cross-corpus asymmetry, so $\Delta_B$ is expected negative (attention drops from the real-boundary level toward or below the within-paragraph level) and $\Delta_C$ positive. For fake-split we additionally report the normalized response $\beta_C$, directly comparable to the $\beta=1$ calibration point, and the percentile rank $C_{\rm pct}$ of the fake-split $\Deff$ within the real boundary-effect distribution (Appendix~\ref{app:stat_methods}).
 
These quantities support two tests: the null of no intervention effect, $H_0^{(1)}:\Delta_B=\Delta_C=0$, satisfied by definition for \textbf{flat} and \textbf{sent\_axial}; and, for models with an active $p_1$ channel, $H_0^{(2)}:\beta_C=1$ for fake-split, whose claim is one of calibrated magnitude against the $\beta=1$ point, whereas fake-merge's is direction and significance, captured by $\Delta_B$ (Appendix~\ref{app:stat_methods}).

We intervene only on $p_1$, holding $p_2$ and $p_3$ fixed, and not at the sentence level, where $\deltap$ and $d$ are harder to disentangle (Section~\ref{sec:exactd}).

\subsection{Data and training}

We evaluate on three structurally distinct corpora: WikiText-2~\citep{merity2016pointer}, OpenWebText~\citep{gokaslan2019openwebtext}, and a collection of Python source files (Code). All models share a fixed architecture---8 layers, 8 heads, $d_{\text{model}}=512$, context length 1024---trained for 5000 steps with three seeds per corpus. Corpus statistics and preprocessing are in Appendix~\ref{app:training}, which also lists the training hyperparameters (Table~\ref{tab:hyper}).

The operative unit throughout is the \emph{paragraph}, defined corpus-specifically: a delimited block of prose in WikiText-2 and OpenWebText, a blank-line-delimited source block in Code. Because these differ in length and internal structure---a heterogeneity the Layer III quantities inherit---we report each quantity with its extraction rule (Appendix~\ref{app:training}).
\section{Layer I: Causal Effect of Hierarchical Position}
\label{sec:layer1}

This section establishes the precondition introduced in the Introduction: before characterizing the response shape, we confirm that hierarchical position affects attention at all. Using the interventions, metric, and bootstrap protocol of Section~\ref{sec:framework} (Appendix~\ref{app:stat_methods}), we ask directly whether attention changes when the token sequence is fixed and paragraph position is not.

\subsection{Single-example demonstration}
\label{sec:workede}

For the single example of Figure~\ref{fig:love-demo}, ``\textit{This is what I love. You are the reason.}'', fake-split (Cond.\ B) gives the second sentence its own $p_1$ with the token sequence unchanged: \textbf{hrope\_axial}'s attention from ``You'' back to the first sentence drops and $\Deff$ across the boundary falls by $3.105$ (a raw change, not the normalized $\Delta$ of Eq.~\eqref{eq:delta}), while the $p_1$-blind controls are unchanged (Figure~\ref{fig:love-demo-full}). This is an illustrative $n=1$ example, not a statistical result.

\subsection{Intervention effects in all three corpora}
\label{sec:workede-main}

Table~\ref{tab:layer1} summarizes the corpus-level results: for each corpus, the fake-merge effect ($\Delta_B$), the fake-split effect ($\Delta_C$), the normalized fake-split response relative to real boundaries ($\beta_C$), and its percentile rank within the real boundary distribution ($C_{\rm pct}$); all use three seeds per corpus (Section~\ref{sec:framework}).

\begin{table*}[!htbp]
\centering
\small
\renewcommand{\arraystretch}{0.9}
\caption{Causal effect of hierarchical paragraph position, per corpus;
quantities defined in Section~\ref{beta-protocol}, brackets are paired
document-cluster bootstrap 95\% CIs, three seeds pooled per corpus
(Appendix~\ref{app:stat_methods}).}
\label{tab:layer1}
\setlength{\tabcolsep}{4pt}
\begin{tabular}{lcccc}
\toprule
Corpus
& $\Delta_B$
& $\Delta_C$
& $\beta_C$
& $C_{\rm pct}$\\
\midrule
WikiText-2
& $-1.300$ ($-1.395,-1.211$)
& $+0.841$ ($0.723,0.964$)
& $1.060$ ($0.985,1.142$)
& $57.9\%$ ($51.7,65.7$)\\
OpenWebText
& $-1.302$ ($-1.372,-1.240$)
& $+0.914$ ($0.835,0.997$)
& $1.069$ ($1.019,1.122$)
& $58.9\%$ ($54.9,62.6$)\\
Code
& $-3.769$ ($-4.526,-3.234$)
& $+3.373$ ($2.800,4.148$)
& $1.522$ ($1.312,1.797$)
& $62.4\%$ ($58.3,66.7$)\\
\bottomrule
\end{tabular}
\end{table*}

The two interventions produce opposite effects: merging two paragraphs compresses attention ($\Delta_B<0$), while splitting dilates it ($\Delta_C>0$).

The compression does not merely stop at the within-paragraph baseline: since $\Delta_B$ is measured against the real-boundary level, the merged response $\beta_{\rm observed}=1+\Delta_B$ overshoots $\beta=0$ in every corpus---mildly for WikiText-2 ($-0.300$) and OpenWebText ($-0.302$), strongly for Code ($-2.769$), whose sentence-boundary level already sits far from the within-paragraph level ($\beta_A^{\rm sent}\approx-1.85$). Since the token sequence is fixed, none of these changes can be attributed to lexical content. The normalized response $\beta_C$ calibrates this directly---values near $1$ mean the fake-split effect is comparable to the average real boundary---and ranges from $1.06$ to $1.52$ across corpora; its 95\% bootstrap interval excludes $\beta_C=1$ for Code ($[1.31,1.80]$), where the fake-split response is significantly \emph{stronger} than the average real boundary, and marginally for OpenWebText ($[1.02,1.12]$), while for WikiText-2 it includes $1$ ($[0.99,1.14]$). The percentile rank $C_{\rm pct}$ places the fake-split response near the middle of the real boundary distribution ($58$--$62\%$), so the manipulation is quantitatively consistent with natural paragraph structure. The $\beta$ and $\Delta$ scales are anchor-normalized and not comparable to the raw depth $U^*$ (Appendix~\ref{app:stat_methods}).

\paragraph{Sanity check.} For \textbf{flat} and \textbf{sent\_axial}, $H_0^{(1)}$ holds by construction: with no $p_1$ channel, falsifying paragraph position changes no input, so $\Delta_B=\Delta_C=0$ exactly. Whether a channel's \emph{content}, rather than its mere presence, drives the response is tested directly in Section~\ref{sec:p1substitution}.

\subsection{Does the channel's content matter, or only its presence?}
\label{sec:p1substitution}

The claim above leaves open a concern: is \textbf{hrope\_axial}'s sensitivity to $p_1$ merely mechanical---the model following whatever coordinate occupies the channel regardless of content? If so, ``hierarchical position is a causal factor'' would reduce to ``the model tracks an arbitrary architectural knob,'' not a claim about what the trained weights represent.

We test this by inference only on the already-trained checkpoints: for \textbf{hrope\_axial}, we replace each input's $p_1$ with one of four substitutes, leaving $p_2$, $p_3$, and the tokens untouched, and measure validation loss (Table~\ref{tab:p1sub}): \textbf{real}, the true paragraph coordinate; \textbf{period}, a mechanical grid $p_1:=\lfloor t/L\rfloor$ at the same density as the true segmentation ($L=$132/102/64 tokens for WikiText-2/Code/OpenWebText); \textbf{rand}, a density-matched random relabeling of $p_1$; and \textbf{const0}, $p_1$ collapsed to a constant, which removes the paragraph axis while retaining $p_2$/$p_3$. The \emph{substitute} \textbf{period} reuses the mechanical grid that defines the trained \textbf{period\_axial} mirror control, which itself is not run through this protocol (Section~\ref{sec:hrope}; Appendix~\ref{app:training}).

\begin{table*}[!htbp]
\centering
\small
\renewcommand{\arraystretch}{0.9}
\caption{Validation loss under $p_1$ substitution, trained \textbf{hrope\_axial} checkpoint, inference only (mean over three seeds). Lower is better.}
\label{tab:p1sub}
\begin{tabular}{lcccc}
\toprule
Corpus & real & period & rand & const0 \\
\midrule
Code & 2.383 & 2.480 & 2.641 & 2.720 \\
WikiText-2 & 5.411 & 5.479 & 5.562 & 5.758 \\
OpenWebText & 5.622 & 5.683 & 5.720 & 5.922 \\
\bottomrule
\end{tabular}
\end{table*}

The ordering $\mathrm{real}<\mathrm{period}<\mathrm{rand}<\mathrm{const0}$ holds in all three seeds of every corpus, and each adjacent comparison has non-overlapping seed ranges with consistent sign across corpora. With three seeds this is directional evidence, not a formal significance test: the minimum one-sided sign-test $p$-value at $n=3$ is $0.125$ (Appendix~\ref{app:seedinference}). Two things follow. First, \textbf{const0} is the worst condition in every corpus by a wide margin ($0.30$--$0.35$ nats worse than real). If sensitivity to $p_1$ were purely mechanical---indifferent to content, requiring only the channel's presence---collapsing $p_1$ to a constant should cost no more than replacing it with any other coordinate; it instead costs markedly more than all three substitutes that retain a coordinate in that slot. Second, \textbf{real} beats \textbf{period} and \textbf{period} beats \textbf{rand} in every seed: period and rand both provide a full, density-matched coordinate at every position, and only their content differs (mechanical periodicity vs.\ true paragraph structure vs.\ density-matched noise). The trained weights discriminate on content, not occupancy.

As a further check, substituting the true paragraph coordinate into the \textbf{period\_axial} checkpoint raises loss only slightly ($+0.015$ to $+0.040$ nats; Appendix~\ref{app:valcost}), far less than the cost \textbf{hrope\_axial} incurs when moved off its own home to period ($+0.061$ to $+0.097$ nats; Table~\ref{tab:p1sub}): a model trained on a purely mechanical coordinate does not become strongly sensitive to true paragraph structure it never saw.

This establishes the paper's first claim---hierarchical paragraph position is a causal factor in attention---without claiming that every Transformer represents paragraphs explicitly, and at little language-modeling cost: never more validation loss than an architecturally identical channel carrying no true positional information, and at most $\approx0.5\%$ above \textbf{flat} (Appendix~\ref{app:valcost}). The real-versus-rand gap in Table~\ref{tab:p1sub} is largest in Code and smallest in OpenWebText ($0.26$, $0.15$, $0.10$ nats for Code, WikiText-2, OpenWebText), and a complementary logistic-regression probe on the residual stream reproduces the same Code~$>$~WikiText-2~$>$~OpenWebText gradient (Appendix~\ref{sec:probing}).
\section{Layer II: Compression Depth as the Corpus-Sensitive Signature}
\label{sec:layer2}

This section answers the Introduction's boxed question, characterizing how strongly attention is compressed near a paragraph boundary and whether that depth exceeds a matched random control; Section~\ref{sec:layer3} then asks whether that depth is explained by corpus statistics.
\subsection{Exact-distance estimation of the compression response}
\label{sec:exactd}

The existence of a causal effect (Section~\ref{sec:layer1}) raises a new question: how strongly does attention depend on paragraph displacement? Throughout, as in the framework (Section~\ref{beta-protocol}), we analyze paragraph-scale displacement only; the protocol does not extend to the sentence level, and neither does any claim below. We write $\deltap=|p_1(i)-p_1(j)|$ and regress the distance-residualized attention $\Deff$ (Eq.~\ref{eq:deff}) on $\deltap$ while controlling for token distance $d$: for each $d$ we contrast each $\deltap$ against the same-$d$ baseline, pool across $d$ by pair-count weighting, and retain only $(d,\deltap)$ cells with at least $2{,}000$ token pairs.\footnote{An earlier, binned version of this estimator introduced a distance-mismatch bias severe enough to flip WikiText-2's sign; Appendix~\ref{app:pooling} details the artifact and this exact-distance fix.}

Let $U(\deltap)$ denote the resulting response curve; we use \emph{compression} for $U(\deltap)<0$ and \emph{dilation} for $U(\deltap)>0$, write $\deltastar$ for the displacement at which an identifiable interior minimum occurs, and $U^*=U(\deltastar)$ for the corresponding depth. Because each token distance $d$ contributes its own weighted estimate, estimating the curve requires no separate binning choice; the remaining modeling choice is the upper range of $d$, fixed at $d\le1{,}024$ to match the training block size (Appendix~\ref{app:pooling}).

Figure~\ref{fig:potential-wells} shows each corpus's compression depth $U^*$ (the value of the fitted degree-2 curve at its vertex) and, where identifiable, the displacement $\deltastar$ at which it occurs. In all three corpora the fitted degree-2 curve has an interior vertex at the reference gate, so $U^*$ is defined as that vertex's value; for WikiText-2, however, the vertex \emph{location} is structurally non-identifiable, so we make no turnover claim for it, though its depth is stable. We therefore treat depth $U^*$, not its location, as the paper's primary quantity, reporting the point estimate with its paired document-cluster bootstrap interval (Table~\ref{tab:equilibrium}, Appendix~\ref{app:wells}; Appendix~\ref{app:binsweep}).

\begin{figure*}[!tbp]
\centering
\includegraphics[width=\textwidth]{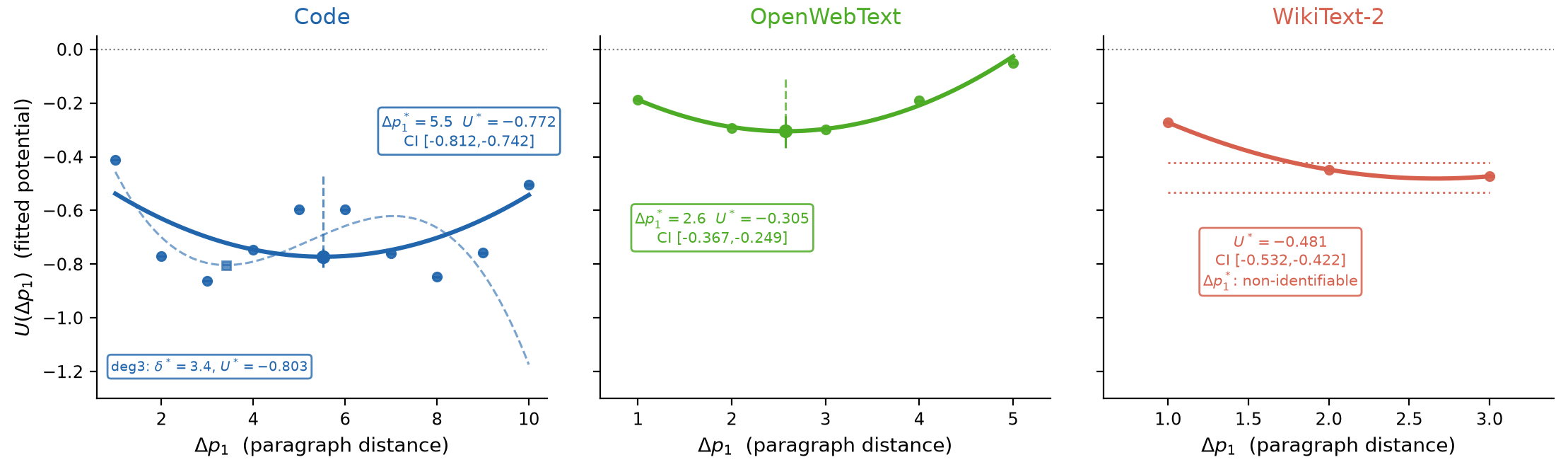}
\caption{Compression response $U(\deltap)$ for \textbf{hrope\_axial}
(exact-distance estimator, $n\ge2{,}000$, three seeds pooled). All three
corpora compress ($U<0$); depth $U^*$ with 95\% CI is annotated per corpus,
and $\deltastar$ is not identifiable for WikiText-2. Numbers:
Table~\ref{tab:equilibrium}.}
\label{fig:potential-wells}
\end{figure*}

All three corpora show compression, and none dilation, anywhere in the sampled range (Figure~\ref{fig:potential-wells}, Appendix~\ref{app:wells}). Compression depth $U^*$ is bootstrap-stable in every corpus---with one caveat, that WikiText-2's depth shifts by $\approx0.1$ across sampling gates, comparable to a single gate's CI width---whereas its displacement $\deltastar$ is not reliably identifiable at our sample sizes: structurally so for WikiText-2, whose long paragraphs (median $\approx122$ tokens) leave too few boundaries per document for the same-$d$ baseline cells to pass the $n\ge2{,}000$ gate beyond $\deltap\approx3$, and fit-sensitively for Code, whose degree-2 ($\approx5.5$) and degree-3 ($\approx3.4$) fits disagree (Appendix~\ref{app:binsweep}). We treat the resulting cross-corpus ordering of $U^*$ (deep to shallow: Code, WikiText-2, OpenWebText) as descriptive rather than inferential (Section~\ref{sec:cross-corpus}).

\subsection{The depth of compression is not exclusive to true paragraph structure}
\label{sec:randcontrol}

The compression above could in principle be a signature of genuine paragraph structure. We test this with \textbf{rand\_axial}, identical to hrope\_axial except that its $p_1$ channel carries a density-matched label $\rho(i)$ resampled independently at every training step, never fixed for a document. Because the label changes every step, the model cannot memorize any position as a boundary; any systematic response can only reflect a generic reaction to a density-matched axial coordinate, not to paragraph content.

Under the same exact-distance protocol, \textbf{rand\_axial} is compressed in every corpus too (Table~\ref{tab:randcontrol}), so compression alone is not diagnostic of true hierarchical structure; what distinguishes the two conditions is depth. We test this with a paired document-cluster bootstrap (same resample indices for both models, controlling for corpus-level sampling variation).

\begin{table*}[!htbp]
\centering
\small
\renewcommand{\arraystretch}{0.9}
\caption{Compression depth for \textbf{rand\_axial} and \textbf{hrope\_axial} (exact-distance protocol, $n\ge2{,}000$), and the significance of their paired difference.}
\label{tab:randcontrol}
\begin{tabular}{lrrrl}
\toprule
Corpus & hrope $U^*$ & rand $U^*$ & $\Delta$(hrope$-$rand) 95\% CI & Significant? \\
\midrule
Code & $-0.772$ & $-0.238$ & $[-0.578,-0.498]$ & \textbf{yes}, hrope deeper \\
WikiText-2 & $-0.481$ & $-0.220$ & $[-0.341,-0.171]$ & \textbf{yes}, hrope deeper \\
OpenWebText & $-0.305$ & $-0.326$ & $[-0.059,+0.093]$ & no (CI includes 0) \\
\bottomrule
\end{tabular}
\end{table*}

In two of three corpora the paired bootstrap interval for the hrope\_axial--rand\_axial depth gap excludes zero. In OpenWebText the two are not resolvable at our seed count: rand\_axial's point estimate ($-0.326$) is marginally more negative than hrope\_axial's ($-0.305$), and the gap interval $[-0.059,+0.093]$ includes zero. We therefore treat the OpenWebText difference as evidence in neither direction on the depth axis, and do not attribute its apparent flattening to a mechanistic cause absent direct evidence: it is a corpus where the depth comparison is unresolvable at our seed count, not one where the causal effect disappears.

rand\_axial's depth also tracks the corpus less strongly than hrope\_axial's (its three depths span $-0.22$ to $-0.33$ versus $-0.31$ to $-0.77$), but the difference is one of \emph{degree}: all six model--corpus cells are compression, and only depth separates the conditions. Both depths are fitted at the interior vertex of the same degree-2 configuration, and we make no claim about rand\_axial's vertex \emph{location} or curve shape. The cross-corpus pattern is descriptive at three corpora, and OpenWebText's gap is unresolved at our seed count, so we cannot separate a true null from a power limitation (Section~\ref{sec:limits}).

\section{Layer III: Comparing Corpus-Only Candidates with the Observed Depth}
\label{sec:layer3}

The previous sections establish the paper's central empirical claim: the paragraph
coordinate causally affects attention, and the depth of the resulting
compression is greater under real structure than under a
density-matched random control, with the paired bootstrap interval
excluding zero in two of three corpora and the third not resolvable at
our seed count
(Section~\ref{sec:randcontrol}). This section takes up the follow-up question posed at the end of
the Introduction: rather than what
mechanistically \emph{causes} the corpus-dependent depth, we ask which
corpus-only candidate quantities, if any, reproduce its observed
ordering.

The question matters because models are trained on text that
is not strictly homogeneous across paragraph boundaries. Adjacent
paragraphs often remain related by topic, discourse, vocabulary, or
style, and this relatedness generally weakens with increasing
separation. If such statistical persistence has a
characteristic scale, it gives a natural corpus-level quantity with
which to compare compression depth $U^*$. Crucially, we compute it
without using any model representations or attention values.

\subsection{Paragraph-level lexical distributions}

For paragraph $p$ with token counts $c_p(w)$ we take the empirical lexical distribution $q_p(w) = c_p(w)/\sum_v c_p(v)$, build a within-document shuffled-order control that preserves the marginal distribution of paragraph content while destroying sequential order, and take the excess sequential dependence of the Jensen--Shannon similarity curve over that control:

\begin{equation}
\Sexc(\delta) = S_{\rm real}(\delta) - S_{\rm shuffle}(\delta),\label{eq:Dcorpus}
\end{equation}

This quantity is entirely corpus-derived. We denote it $\Sexc$, not a bare $D$, to avoid confusion with the attention-derived $\Deff$ (Eq.~\ref{eq:deff}); the two are unrelated except that both measure residual dependence after removing a baseline, and the similarity curve, divergence, and shuffled-order control are defined in full in Appendix~\ref{app:corpusscale}. We fit the observed decay by
\begin{equation}
\Sexc(\delta) = S_\infty + A\exp\left(-\frac{\delta}{\lstruct}\right),\label{eq:lambda}
\end{equation}
and interpret $\lstruct$ as a characteristic \emph{paragraph-level
structural persistence length}: a Jensen--Shannon, distributional measure
of paragraph-level coherence, not a mutual-information decay length. All
corpus-level estimates in this section use document-level bootstrap
resampling ($500/500$ replicates completed for the main analyses;
$\lstruct$ is summarized by its bootstrap median, because its point
estimate is unstable; Section~\ref{sec:layer3-main}), which
quantifies uncertainty from the available corpus sample but not
uncertainty in the model-derived $\deltastar$, treated here as an
independently obtained target (Section~\ref{sec:layer2};
Appendix~\ref{app:stat_methods}). Fitted values and the remaining estimators ($\delta_{1/2}$,
$\llocal$, $L_{\rm int}$, $\delta_{\rm zero}$) are in
Appendix~\ref{app:corpusscale}.

\subsection{Main comparison}
\label{sec:layer3-main}

Table~\ref{tab:corpus} (Appendix~\ref{app:corpusscale}) reports the resulting values against compression depth $U^*$ (Table~\ref{tab:equilibrium}), which, unlike $\deltastar$, is defined and bootstrap-stable in all three corpora: $\lstruct$ (bootstrap median) is 4.85 ([3.80,6.70]), 3.74 ([2.94,5.09]), and 3.85 ([2.49,8.19]) for Code, WikiText-2, and OpenWebText, against $U^*$ of $-0.772$, $-0.481$, and $-0.305$. Since $\lstruct$ is
corpus-only and independent of the model-fitting protocol, the comparison
covers all three corpora---and the orderings do not agree. Ranked by value, $\lstruct$ gives
WikiText-2 $<$ OpenWebText $<$ Code, whereas $U^*$ (more
negative $=$ deeper) is, deep to shallow, Code,
WikiText-2, OpenWebText: nominally the reverse order on the two prose corpora.
Code occupies the same rank (largest $\lstruct$, deepest compression) in
both, but WikiText-2 has the \emph{smallest} $\lstruct$ yet compresses
\emph{more deeply} than mid-ranked OpenWebText. A corpus-only predictor should get all three ranks right, not just the
endpoint; $\lstruct$ gets only Code right. The prose-corpus comparison is
also fragile (medians $0.11$ apart, overlapping intervals, and point
estimates $5.48/4.90/3.37$ that order the corpora differently), so
$\lstruct$ cannot separate WikiText-2 from OpenWebText. Defined independently of any trained model, it therefore does
not reproduce the ordering of the learned compression depth: a negative
finding.

\subsection{Further candidates: at best partial}
\label{sec:alt-scales}
\label{sec:coherence-candidates}

We compare eight corpus-only quantities---decomposing into three
constructs rather than eight independent hypotheses: \emph{lexical
persistence} (five variants built from the same decay curve underlying
$\lstruct$), \emph{paragraph length}, and \emph{embedding-based
coherence} (two summaries of one coherence curve)---against the ordering
of compression depth $U^*$ (deep to shallow: Code $-0.772$, WikiText-2
$-0.481$, OpenWebText $-0.305$; a descriptive ordering of fixed corpora,
Section~\ref{sec:cross-corpus}). Each quantity is defined and estimated
in Appendix~\ref{app:alt-scales}, and Table~\ref{tab:allcandidates}
reports all eight values grouped by construct, each candidate's
cross-corpus differences judged against its own variance estimate where
available.

\begin{table*}[!htbp]
\centering
\footnotesize
\renewcommand{\arraystretch}{0.9}
\caption{Eight corpus-only quantities by construct, against the ordering
of compression depth $U^*$ (deep to shallow: Code, WikiText-2,
OpenWebText). ``Sep.\ pairs'': pairwise corpus comparisons separated at
$p<0.05$ under that candidate's own variance estimate, of three; the
reference row is descriptive (Section~\ref{sec:cross-corpus}).}
\label{tab:allcandidates}
\setlength{\tabcolsep}{4pt}
\begin{tabular}{lrrrcc}
\toprule
Candidate & Code & WikiText-2 & OpenWebText & Ordering matches? & Sep.\ pairs \\
\midrule
$U^*$ (reference) & $-0.772$ & $-0.481$ & $-0.305$ & --- & \emph{descriptive} \\
\midrule
\emph{Lexical persistence (five variants):} $\lstruct$ & 4.85 & 3.74 & 3.85 & No & n/a \\
$\delta_{1/2}$ (rescaling of $\lstruct$) & 3.36 & 2.59 & 2.67 & No & n/a \\
$\llocal$ & 2.22 & 2.16 & 1.38 & Yes & 0/3 \\
$L_{\rm int}$ & 3.84 & 3.02 & 3.30 & No & n/a \\
$\delta_{\rm zero}$ (negative control) & 5.17 & 6.93 & 3.72 & No & n/a \\
\midrule
\emph{Paragraph length:} median (tokens) & 55 & 122 & 45 & No & n/a \\
\midrule
\emph{Embedding coherence:} aggregate & 0.096 & 0.066 & 0.051 & \textbf{Yes} & \textbf{2/3} \\
Coherence half-life (paragraphs) & 2.53 & 2.47 & 2.18 & Yes & 0/3 \\
\bottomrule
\end{tabular}
\end{table*}

\paragraph{Verdicts.} Lexical persistence fails: $\delta_{1/2}$ shares
$\lstruct$'s ordering by construction and is not a second data point,
$\lstruct$, $L_{\rm int}$, and $\delta_{\rm zero}$ each swap a corpus
pair, and $\llocal$'s ordering matches but separates no pair even under
its own bootstrap uncertainty. Paragraph length also fails, swapping Code
and WikiText-2. Embedding coherence comes closest: its two summaries
agree directionally but differ in strength, with aggregate coherence
separating Code from both other corpora yet not WikiText-2 from
OpenWebText, and coherence half-life reproducing the direction with wider
intervals (e.g.\ Code $[1.95,3.18]$) that separate no pairwise difference;
a vocabulary-robustness check (Appendix~\ref{app:vocab-robust}, Table~\ref{tab:vocab}) confirms
that $\lstruct$'s ordering is unchanged at every tested cutoff, so this
failure is a mismatch between $\lstruct$ and $U^*$ rather than an artifact
of one vocabulary size.
Counting constructs rather than table rows, no corpus-only structural
scale among these three fully reproduces the ordering of compression
depth, and embedding coherence is the candidate we would recommend for
follow-up. With only three corpora, any cross-corpus ordering---of $U^*$
or of a candidate explanation---is anecdotal in the statistical sense
(Section~\ref{sec:limits}).

\subsection{Moving the analysis down to the paragraph level}
\label{sec:parareg}

The candidate comparison of Section~\ref{sec:alt-scales} is corpus-level by
construction, so we re-ask it one level down: within a
corpus, is a paragraph's own compression depth $U_p$ a function of its own
observable properties? This trades three corpus rows for nearly three
thousand paragraph observations and, more importantly, for a falsifiable
form---a universal relation should reproduce as a per-paragraph slope that
\emph{agrees} across corpora, a corpus- or volume-driven one should leave
those slopes \emph{disagreeing}.

\textbf{Per-paragraph depth and predictors.} We fix the shallowest
cross-paragraph displacement, $\deltap{=}1$, and score each boundary
pair's log-attention relative to a within-paragraph baseline matched on
both exact token distance $d$ and query-context band $b(i)$; averaging
over the $\deltap{=}1$ pairs incident to a paragraph gives its depth
$U_p$ (Eqs.~\ref{eq:upair}--\ref{eq:up}), with the same sign convention as
the corpus-level $U^*$: negative $U_p$ is compression, so depth is the
magnitude of a negative value. For each paragraph we record its token
length, its type--token ratio, its embedding coherence with the adjacent
paragraphs of the same document (mean all-MiniLM-L6-v2 cosine,
corpus-derived and seed-invariant), and its fractional position
$pos/N_p$, and fit a weighted least squares model of $U_p$ on these four
within each corpus, weighted by the paragraph's incident
$\deltap{=}1$ pair count (Appendix~\ref{app:parareg}).

\begin{table*}[!htbp]
\centering
\small
\renewcommand{\arraystretch}{0.9}
\caption{Paragraph-level compression regression of $U_p$
(Eq.~\ref{eq:up}), weighted least squares within each corpus. Left:
standardized coefficients ($^{*}$\,$p<0.05$, $^{**}$\,$p<0.01$,
$^{***}$\,$p<0.001$). Right: $|z|$ for pairwise differences of raw slopes
(C, W, O: Code, WikiText-2, OpenWebText); $|z|>2.87$ is significant after
Bonferroni correction over twelve comparisons (Appendix~\ref{app:parareg}).}
\label{tab:parareg}
\begin{tabular}{lcccccc}
\toprule
 & \multicolumn{3}{c}{$\beta_{\rm std}$, per corpus}
 & \multicolumn{3}{c}{slope difference $|z|$}\\
\cmidrule(lr){2-4}\cmidrule(lr){5-7}
Predictor & Code & Wiki & OpenWebText & C--W & C--O & W--O\\
\midrule
Length & $-0.044^{***}$ & $-0.016^{**}$ & $+0.005$ & 2.97 & 7.44 & 3.39\\
Type--token ratio & $+0.052^{***}$ & $-0.007$ & $-0.004$ & 4.09 & 5.38 & 0.44\\
Embedding coherence & $+0.015^{**}$ & $+0.013^{**}$ & $+0.017^{***}$ & 0.03 & 0.93 & 0.95\\
Position (control) & $-0.012^{*}$ & $-0.123^{***}$ & $-0.099^{***}$ & 17.91 & 14.19 & 6.37\\
\midrule
$N$ paragraphs & 1{,}167 & 438 & 1{,}285 & & &\\
Weighted $R^2$ & 0.287 & 0.693 & 0.442 & & &\\
\bottomrule
\end{tabular}
\end{table*}

\textbf{Results.}
Of the four predictors, exactly one transfers: the embedding-coherence
slope is positive and significant within every corpus (standardized
Code $+0.015$, WikiText-2 $+0.013$, OpenWebText $+0.017$; all $p<0.01$),
and its raw slopes ($0.079$, $0.078$, $0.109$) show no detectable
difference between any pair of corpora ($|z|\le0.95$, $p>0.3$ for all
three pairs; Table~\ref{tab:parareg}). Its direction, however, runs
opposite to the corpus-level ranking of Table~\ref{tab:allcandidates}:
under the sign convention of Eq.~\ref{eq:up} a positive coefficient means
a \emph{smaller} depth magnitude, so---holding length, diversity, and
position fixed---paragraphs whose embedding neighbors are more similar
are compressed \emph{less} deeply, at a rate common to all three corpora.
What transfers is a conditional slope, holding the other three
predictors fixed, not a bivariate relation. Every other predictor is corpus-specific: each has
at least one pair of corpora whose slopes differ beyond the Bonferroni
threshold (position up to $|z|=17.9$, length up to $7.4$, type--token
ratio up to $5.4$), with Code most often the outlying corpus, and Code
also retains the lowest explained variance ($R^2=0.287$). We read this as
a refinement, not a rescue, of Section~\ref{sec:alt-scales}: no new
corpus-level scale is introduced, so the negative finding of
Section~\ref{sec:alt-scales} stands.

\textbf{Is the compression uniform or hub-driven?} An exploratory masking probe finds that the compression depth of Section~\ref{sec:exactd} tracks the attention mass removed rather than which tokens carry it, with a residual beyond this mass effect of similar size across corpora ($14$--$18$ percentage points of the well); it changes no main-text estimate (Appendix~\ref{app:hubmask}).
\section{Limitations}
\label{sec:discussion}
\label{sec:limits}

\paragraph{No corpus-only predictor explains the corpus-dependent depth.} None of the three constructs we tested reproduces the cross-corpus ordering of $U^*$ or predicts why the real-versus-random gap is resolvable in two corpora but not the third; this rules out the accounts tested, not the possibility that some other corpus-level property determines depth. The paragraph-level probe is likewise within-corpus and correlational---a per-paragraph slope of $U_p$ on paragraph properties, at $\deltap{=}1$ only---so it identifies no mechanism and supplies no corpus-level scale.

\paragraph{Location is not always a well-posed question.} $\deltastar$ is structurally non-identifiable for WikiText-2 and fit-sensitive for Code (Section~\ref{sec:exactd}), which is why depth, not location, is the paper's primary quantity.

\paragraph{Scope and measurement.} Three corpora limit any cross-corpus claim to the descriptive; paragraphs are the relevant unit only under a corpus-specific definition (Section~\ref{sec:framework}); ``geometry'' is used descriptively; the check of Appendix~\ref{app:two-process} excludes only the additive competing-process family; and all experiments use 8-layer, $d_{\rm model}=512$ models trained for 5000 steps, with no downstream-task metric reported. Two measurement limits bound the protocol: \textbf{period\_axial} (Section~\ref{sec:hrope}) makes $p_1$ an almost deterministic function of token distance, so a coordinate not decoupled from token distance cannot be analyzed under it at all, and three seeds bound every seed-level comparison to directional evidence (Appendix~\ref{app:extlimits}).

\paragraph{Cross-corpus ordering is descriptive.}\label{sec:cross-corpus}
Because the document-level bootstrap quantifies \emph{within-corpus} uncertainty and the three corpora are fixed by choice rather than drawn from a common population, the depth ordering (deep to shallow: Code, WikiText-2, OpenWebText) establishes that the corpora differ, not that the specific order generalizes; we avoid ``statistically significant'' language for cross-corpus comparisons.

\section{Conclusion}
\label{sec:conclusion}

Reading order cannot by itself determine hierarchical structure: the same token sequence admits exponentially many paragraph segmentations (Section~\ref{sec:intro}). Given an explicit hierarchical coordinate, the model uses it: attention is compressed near paragraph boundaries in every corpus tested, and this is a causal effect of hierarchical position, not an artifact of token content, since fixing the token sequence and changing only $p_1$ changes attention (Section~\ref{sec:layer1}).

Compression alone, however, is not evidence of genuine structure: an architecturally identical channel with density-matched random labels compresses too. Nor is the \emph{location} of deepest compression, which is not even identifiable in one of three corpora. What distinguishes real structure is \emph{depth}: real structure compresses more deeply where the comparison is resolvable (two of three corpora), and its depth differs across corpora (Section~\ref{sec:layer2}). Where an interior turning point is identifiable (Code, OpenWebText), it is not a trivial superposition of two monotone processes (Appendix~\ref{app:two-process}).

What sets this depth remains open. No corpus-only statistic we tested---across lexical persistence, paragraph length, and embedding-based coherence---reproduces its cross-corpus ordering, though embedding-based coherence comes closest; at the paragraph level, the within-corpus coherence slope is common across corpora but runs opposite to the corpus-level ranking, and no corpus-level scale is restored (Section~\ref{sec:layer3}). Reading-order distance is not structural distance; the depth of compression is the result a reader should take from this paper, and explaining it---including why the real-versus-random gap resolves in two corpora but not the third---is the open problem we leave (Section~\ref{sec:limits}).

\bibliography{refs}
\appendix
\onecolumn
\section{Statistical Methods}
\label{app:stat_methods}

\subsection{Anchor and $\beta$ estimation}
$\mathrm{anchor}_{\rm within}$ and $\mathrm{anchor}_{\rm para}$ are
computed from the unmanipulated portion of the same document sample used
for the fake-$p_1$ interventions, ensuring anchors and manipulated points
are drawn from a common distribution of documents within each bootstrap
resample.

\paragraph{Anchor-normalized responses are not comparable to raw
depth.} Because the $\beta$ and $\Delta$ scales of
Section~\ref{beta-protocol} are normalized to the interval between the
within-paragraph and real-boundary anchors, their magnitudes are not
directly comparable to the raw-scale depth $U^*$ of
Section~\ref{sec:exactd}: Code's fake-merge effect $\Delta_B$ is roughly three times
WikiText-2's ($-3.77$ vs.\ $-1.30$), whereas its
depth is only about $1.6\times$ larger ($-0.772$ vs.\ $-0.481$). The
former are anchor-normalized responses, the latter a raw attention
residual.

\paragraph{Normalized response, percentile rank, and the two tests.}
For fake-split we additionally report the normalized response
\begin{equation}
\beta_C := \beta_{\rm observed}\big|_{\mathcal M=\text{fake-split}},
\end{equation}
directly comparable to the $\beta=1$ calibration point: $\beta_C=1$
means the fake boundary produces exactly the average real-boundary
magnitude; and $C_{\rm pct}$, the percentile rank of the fake-split
$\Deff$ within the real boundary-effect distribution. These quantities
support two tests: the null of no intervention effect,
$H_0^{(1)}:\Delta_B=\Delta_C=0$, satisfied by definition for
\textbf{flat} and \textbf{sent\_axial}, since without a $p_1$ channel
falsifying paragraph position leaves $\Deff$ unchanged; and, for models
with an active $p_1$ channel, $H_0^{(2)}:\beta_C=1$ for fake-split. A
normalized response is reported only for fake-split, whose claim is one
of calibrated magnitude and needs the $\beta=1$ reference point;
fake-merge's claim is one of direction and significance, fully captured
by $\Delta_B$.

\subsection{Joint cluster bootstrap}
Each of $B=4000$ resamples draws documents with replacement (not
individual points), and recomputes $\overline{\Deff(\mathcal M)}$,
$\mathrm{anchor}_{\rm within}$, and $\mathrm{anchor}_{\rm para}$ from the
same resampled document set before forming $\beta^{(b)}$.

The signed effects reported in Table~\ref{tab:layer1} are second-order
quantities of the same bootstrap: $\Delta_B=\beta_B-1$ (since
$\beta_A^{\rm para}\equiv1$ in every resample, $\Delta_B$'s interval is
the $\beta_B$ interval shifted by $-1$) and
$\Delta_C=\beta_C-\beta_A^{\rm sent}$, whose interval conservatively
spans the paired combination of the $\beta_C$ and $\beta_A^{\rm sent}$
bootstrap bounds.

\subsection{Equilibrium point estimation}
The exact-distance response curve $U(\deltap)$ is estimated per token
distance $d$ ($d\le1{,}024$): for each $d$ individually, the contrast at
each paragraph displacement $\deltap$ is computed relative to the
same-$d$, same-paragraph ($\deltap=0$) baseline, retaining only
$(d,\deltap)$ cells with at least $n\ge2{,}000$ token pairs, and the
cell-level estimates are then combined across $d$ by pair-count
weighting (this exact-distance protocol is defined in
Appendix~\ref{app:pooling}). For each corpus we fit weighted degree-2 and
degree-3 polynomials to $U(\deltap)$, solve for interior stationary
points of each fitted curve, and report the degree-2 fit as the headline
and the degree-3 fit as a robustness check for the location parameter
(Code's $\deltastar$ differs between them). Confidence intervals use a
paired 2000-draw document-level cluster bootstrap, in which identical
resample indices are applied to the two conditions being compared (e.g.,
hrope\_axial vs.\ rand\_axial, original vs.\ shuffled); a corpus is
reported as having no interior equilibrium when the fitted curve is
monotonic on the sampled range.

\subsection{Corpus-only candidate quantities and their per-variant verdicts}
\label{app:alt-scales}

This appendix gives the definitions behind the eight corpus-only
quantities compared in Section~\ref{sec:alt-scales} and
Table~\ref{tab:allcandidates}, together with the per-variant verdicts
summarized there.

\paragraph{Definitions.} The set includes $\lstruct$ from
Section~\ref{sec:layer3-main} plus seven further quantities.
\emph{Lexical persistence} contributes five variants built from the
same decay curve $R(\delta)$ underlying $\lstruct$: the half-decay
point $\delta_{1/2}=\lstruct\ln2$ (an exact rescaling of $\lstruct$,
sharing its ordering by construction, not an independent statistic); a
local decay scale $\llocal$ fit from the first three sampled
displacements only; an integrated persistence
$L_{\rm int}=\sum_{\delta=1}^{10}R(\delta)$; and a negative control,
the zero-crossing scale $\delta_{\rm zero}$ (the displacement at which
paragraph pairs become statistically indistinguishable from
shuffled-order pairs). \emph{Paragraph length} contributes median
paragraph length. \emph{Embedding-based coherence} contributes two
summaries of one coherence curve: an aggregate score (mean
adjacent-paragraph cosine similarity minus a random-pair baseline) and
the curve's half-life in paragraphs. Each candidate's cross-corpus
differences are judged against its own variance estimate where
available; the reference row $U^*$ is descriptive
(Section~\ref{sec:cross-corpus}).

\paragraph{Per-variant verdicts.} Within lexical persistence,
$\delta_{1/2}$ shares $\lstruct$'s ordering by construction and is not a
second data point; of the remaining four, $\lstruct$, $L_{\rm int}$,
and $\delta_{\rm zero}$ do not match the observed ordering
($\lstruct$ and $L_{\rm int}$ swap WikiText-2 and OpenWebText,
$\delta_{\rm zero}$ swaps Code and WikiText-2), and $\llocal$'s
ordering matches but its cross-corpus differences are separated for no
pair even under its own within-corpus bootstrap uncertainty, with the
largest point contrast (Wiki vs.\ OpenWebText) only marginal.
Lexical persistence is therefore a single failed construct, not four or
five. Paragraph length also fails (it swaps Code and WikiText-2).
Within embedding coherence the two summaries agree directionally but
differ in strength: aggregate coherence separates Code from both other
corpora under its own within-corpus bootstrap but not WikiText-2 from
OpenWebText (a pair $U^*$ separates only descriptively,
Section~\ref{sec:layer3-main}), while coherence
half-life reproduces the direction with wider intervals (e.g.\ Code
$[1.95,3.18]$) that separate no pairwise difference. This is one
construct---embedding coherence---coming closer than the other two, with
aggregate coherence its more reliable summary, not two independent
successes.

\subsection{Corpus-only structural scale estimation}
\label{app:corpusscale}
\paragraph{The lexical similarity curve.} For paragraph $p$, let
$c_p(w)$ denote the count of token $w$ in that paragraph, so the
empirical lexical distribution is $q_p(w) = c_p(w)/\sum_v c_p(v)$. For
two such distributions $q_p$ and $q_{p+\delta}$ we form the midpoint
$m = \frac12(q_p+q_{p+\delta})$ and the Jensen--Shannon divergence
\[
\JS(q_p,q_{p+\delta}) = \frac12\KL(q_p\Vert m) + \frac12\KL(q_{p+\delta}\Vert m),
\]
which we normalize into a similarity score
$S(p,p+\delta) = 1-\JS(q_p,q_{p+\delta})/\log 2$ that is $1$ for
identical paragraph distributions and $0$ when they are maximally
different. Averaging over paragraph pairs gives the sequential
similarity curve
\[
S_{\rm real}(\delta) = \E_p[S(p,p+\delta)],
\]
and the shuffled-order control $S_{\rm shuffle}(\delta)$ is obtained by
randomly permuting paragraph order within each document before
recomputing the pair similarity at displacement $\delta$, holding the
marginal distribution of paragraph content fixed while destroying
sequential order. The excess sequential dependence of
Eq.~\eqref{eq:Dcorpus} is their difference
$\Sexc(\delta)=S_{\rm real}(\delta)-S_{\rm shuffle}(\delta)$, fitted by
the decay of Eq.~\eqref{eq:lambda} with characteristic scale
$\lstruct$.
All corpus-only statistics ($\lstruct$, $\delta_{1/2}$,
$\llocal$, $L_{\rm int}$, $\delta_{\rm zero}$) are computed from
$\Sexc(\delta)=S_{\rm real}(\delta)-S_{\rm shuffle}(\delta)$ and use the same
document-level cluster bootstrap protocol as the model-derived estimates.

\begin{table*}[!htbp]
\centering
\caption{
Corpus-only structural scale compared with the model-derived
compression depth $U^*$ (exact-distance protocol, Table~\ref{tab:equilibrium});
$\lstruct$ is the bootstrap median with its 95\% CI.
}
\label{tab:corpus}
\begin{tabular}{lrrr}
\toprule
Corpus
&
$U^*$
&
$\lstruct$
&
95\% CI
\\
\midrule
Code
& $-0.772$
& 4.85
& [3.80,6.70]
\\
WikiText-2
& $-0.481$
& 3.74
& [2.94,5.09]
\\
OpenWebText
& $-0.305$
& 3.85
& [2.49,8.19]
\\
\bottomrule
\end{tabular}
\end{table*}

\subsection{A distance-bin pooling artifact, and why we moved to exact-distance estimation}
\label{app:pooling}

Earlier versions of this analysis controlled for token distance $d$ by
pooling token pairs into a small number of bins ($d\le64$, $d\le128$,
$d\le256$), computing $\Deff$ relative to each bin's own $\deltap=0$
baseline, and combining bins by pair-count weighting. We identify here
a specific artifact in this procedure, severe enough to have reported
a spurious sign for one corpus, and describe the exact-distance
estimator we replaced it with (used throughout Section~\ref{sec:layer2} onward).

\paragraph{The artifact.} Within a bin, the mean token distance of the
$\deltap=k$ group of pairs need not match that of the $\deltap=0$
baseline group: because larger $\deltap$ pairs are, on average, found
at larger $d$ within a wide bin, the two groups being contrasted are
implicitly compared at different typical distances. Since $\Deff$
itself depends on $d$, this mismatch injects a bias into the pooled
contrast whose sign and size depend on how mismatched the two groups
happen to be. For WikiText-2 at $\deltap=1$, we traced this
concretely: every individual token distance $d\in\{1,\ldots,256\}$
shows compression ($U<0$) at $\deltap=1$, in every seed, with no
exception; yet the $d\le256$-pooled estimate showed \emph{dilation}
($U\approx+0.11$). The $\deltap=1$ group's pairs in that bin sit at a
mean $d\approx184.5$, versus $\approx167.0$ for the $\deltap=0$
baseline group; because attention itself varies with $d$ within the
bin, the two groups are contrasted at different typical distances and
the pooled contrast is biased, which the exact-distance estimator below
removes by never comparing pairs at different $d$. Code was largely unaffected by the same
procedure, because its two groups' $d$-distributions happened to be
well matched within the bin (a difference of $0.4$, versus WikiText-2's
$17.5$) --- the artifact is a property of how mismatched the groups
are, not of the pooling procedure applied uniformly across corpora.

\paragraph{Exact-distance estimation.} We replace bin pooling with an
estimator that never compares pairs at different $d$: for each token
distance $d$ individually, we compute the contrast at each $\deltap$
relative to the \emph{same-$d$} $\deltap=0$ baseline, retain only
$(d,\deltap)$ cells with at least $n=2{,}000$ token pairs, and combine
across $d$ by pair-count weighting. Within each $(d,\deltap)$ cell the
contrast is the simple, unweighted mean over that cell's token pairs.
We compared three ways of weighting these per-$d$ cell contrasts when
combining them across $d$ --- scheme A (pair-count, $w_d\propto n(d,\deltap)$),
scheme B (inverse-variance, $w_d\propto1/\mathrm{se}^2(d,\deltap)$), and
scheme C (uniform in $d$, $w_d\propto1$) --- which agree to within
$\approx0.03$ of $U^*$ at the reference gate in all three corpora; we
report scheme A throughout the main text (the A--C labels are local to
this appendix and are not used elsewhere). The $n\ge2{,}000$ setting is
the smallest threshold at which the spurious diagonal band is absent in
\emph{all three} corpora: at lower thresholds an apparent diagonal band
of anomalous cells appears in $(d,\deltap)$ heatmaps (at $n\ge200$, for
Code, $32$ spurious positive-$U$ cells, traced to the ``resolvability
frontier'' where only the sparsest, noisiest cells for a given
$\deltap$ survive the gate); raising the threshold to $n\ge1{,}000$
reduces this to $12$ cells, and $n\ge2{,}000$ eliminates the band
entirely ($0$ cells), while WikiText-2 shows no such band at any
threshold. The band is a cell-level heatmap artifact: its spurious
cells are the sparsest to survive the gate and receive correspondingly
little weight under pair-count pooling, so it does not move the pooled
depth --- Code's $U^*$ at $n\ge200$ ($-0.772$) already equals its
$n\ge2{,}000$ value, and this is why Table~\ref{tab:thresholds} can
report the $n\ge200$ depth for Code without contradiction.
All compression-depth estimates reported in the main text use this
$n\ge2{,}000$ protocol; Appendix~\ref{app:binsweep}
(Table~\ref{tab:thresholds}) additionally reports the same estimator at
other gates as a robustness check.

\paragraph{Why position, but not depth, remains sensitive to the sampling threshold.} Raising the $n$ threshold further to test convergence reveals a second, distinct issue specific to $\deltastar$'s \emph{location}: for WikiText-2, the fitted vertex $\deltastar$ drifts monotonically downward as the threshold rises ($4.11\to2.92\to2.70\to2.67\to2.51$ across $n\ge200,1{,}000,1{,}500,2{,}000,3{,}000$) and never plateaus, because the $\deltap=0$ baseline cells resolvable at higher thresholds are themselves confined to smaller $d$, which shifts which points the parabola is fit to. Compression depth $U^*$ moves far less: it shifts by $\approx0.1$ over the same range ($-0.577$ at $n\ge200$ to $-0.475$ at $n\ge3{,}000$; e.g., $[-0.532,-0.422]$ at $n\ge2{,}000$), an order of magnitude smaller in proportional terms than the location drift. This is why Section~\ref{sec:exactd} reports $\deltastar$ as structurally non-identifiable for WikiText-2 while still reporting $U^*$ as its primary, stable quantity.

Pooling distance bins beyond $d=1{,}024$ into the same fit remains unsound for a further reason, independent of the artifact above: at these longer ranges, paragraph displacement $\deltap$ and token distance $d$ become approximately collinear (mean $d$ increases monotonically and substantially with $\deltap$), which reintroduces exactly the token-distance confound the exact-distance protocol is designed to remove. We therefore report far-distance statistics only as a separate, non-fitted diagnostic (Appendix~\ref{app:binsweep}), never combined with the $d\le1{,}024$ estimates in Table~\ref{tab:equilibrium}.

\subsection{Robustness of compression depth and location to the sampling threshold}
\label{app:binsweep}

The exact-distance protocol (Section~\ref{sec:exactd}, Appendix~\ref{app:pooling}) requires choosing a minimum per-cell sample size $n$; we use $n\ge2{,}000$ throughout the main text, the smallest threshold at which the spurious diagonal band of Appendix~\ref{app:pooling} is absent in all three corpora. Table~\ref{tab:thresholds} reports compression depth $U^*$ across thresholds; here we describe how depth and location respond to that choice, and what happens beyond the $d\le1{,}024$ range used in the main analysis.

\begin{table*}[!htbp]
\centering
\footnotesize
\renewcommand{\arraystretch}{0.9}
\caption{Compression depth $U^*$ across the sampling gate $n$ (exact-distance estimator, same-$d$ cells, three seeds pooled per corpus, pair-count-weighted across $d$, window $d\le1{,}024$). All entries are re-pooled from the same stored cell counts and estimator as the main text (Section~\ref{sec:exactd}); bootstrapped 95\% CIs (Section~\ref{sec:exactd}) are reported only at the reference gate $n\ge2{,}000$. Within each corpus the depth is stable from $n\ge1{,}000$ onward (Code stays at $-0.77$, WikiText-2 moves $-0.501\to-0.481\to-0.475$, OpenWebText stays at $\approx-0.30$). The $n\ge200$ values are less reliable: Code is already converged ($-0.772$), WikiText-2's $-0.577$ is $0.1$ from its $n\ge1{,}000$ value, and OpenWebText's is omitted ($---$) because at $n\ge200$ neither of the weighting schemes of Appendix~\ref{app:pooling} produces a depth comparable to the rest of the table: under scheme B (inverse-variance) the fitted degree-2 curve has no interior well at this gate, and under scheme A (pair-count) it gives $-0.231$, an outlier far from the $n\ge1{,}000$--$3{,}000$ plateau ($\approx-0.30$).}
\label{tab:thresholds}
\begin{tabular}{lcccc}
\toprule
Corpus & $n\ge200$ & $n\ge1{,}000$ & $n\ge2{,}000$ & $n\ge3{,}000$ \\
\midrule
Code & $-0.772$ & $-0.773$ & $-0.772$ [$ -0.812,-0.742$] & $-0.772$ \\
WikiText-2 & $-0.577$ & $-0.501$ & $-0.481$ [$ -0.532,-0.422$] & $-0.475$ \\
OpenWebText & --- & $-0.295$ & $-0.305$ [$ -0.367,-0.249$] & $-0.310$ \\
\bottomrule
\end{tabular}
\end{table*}

All rows of Table~\ref{tab:thresholds} are re-pooled from the stored per-document cells with the main-text estimator (pair-count-weighted, $d\le1{,}024$), so every cell is a legitimate exact-distance estimate; the one cell marked ``---'' (OpenWebText at $n\ge200$) is omitted for the reason given in the caption. Between $n\ge1{,}000$ and $n\ge2{,}000$, Code's depth moves by $0.001$ ($-0.773\to-0.772$) and OpenWebText's by $0.01$ ($-0.295\to-0.305$); we do not claim an onset at $n\approx1{,}000$ for these two corpora. WikiText-2 differs: its depth $U^*$ moves monotonically from $-0.577$ at $n\ge200$ to $-0.475$ at $n\ge3{,}000$, a shift of $\approx0.1$ that no single gate's bootstrap CI would resolve (e.g.\ $[-0.532,-0.422]$ at $n\ge2{,}000$), whereas its fitted location $\deltastar$ drifts monotonically downward and never plateaus ($4.11\to2.92\to2.70\to2.67\to2.51$ across $n\ge200,1{,}000,1{,}500,2{,}000,3{,}000$), and beyond $n\approx3{,}400$ the fit becomes underdetermined (too few surviving $\deltap$ channels to fit a degree-2 polynomial at all). The cause is structural, not numerical: the $\deltap=0$ baseline cells surviving a higher threshold are increasingly confined to smaller $d$ (WikiText-2's longest paragraph is $465$ tokens, so $\deltap=0$ pairs cannot exist beyond that separation, and only $4{,}673$ such pairs in the entire corpus exceed $d=400$), which shifts the points the parabola is fit against whenever the threshold changes; there is no stable vertex to converge to within the range this corpus's paragraph lengths allow. Thus Table~\ref{tab:thresholds} reports compression depth $U^*$ --- the depth at the fitted vertex of the same degree-2 curve used throughout (pair-count-weighted, $d\le1{,}024$) --- as a stable, primary quantity for this corpus.

\paragraph{Far-distance range ($d>1{,}024$).} We do not extend the exact-distance protocol beyond $d=1{,}024$: paragraph displacement $\deltap$ and token distance $d$ become increasingly collinear at longer range, which would reintroduce the confound the exact-distance protocol is designed to remove (Appendix~\ref{app:pooling}). All estimates in this paper are therefore confined to $d\le1{,}024$, matching the training block size.

\paragraph{Why WikiText-2's location is harder to pin down than Code's or OpenWebText's.} A plausible qualitative account is paragraph length: WikiText-2's median paragraph is $122$ tokens, roughly $2$--$3\times$ Code's ($55$) or OpenWebText's ($45$) (Section~\ref{sec:coherence-candidates}), so within any fixed token-distance window fewer WikiText-2 boundaries occur at all, leaving fewer, sparser $\deltap=0$ baseline cells to anchor a fit at large $d$. We have not verified this against the drift pattern above, so we offer it as a hypothesis, not an established explanation.

\paragraph{Corpus alignment.} All hub-masking estimates (Appendix~\ref{app:hubmask}) use the same per-corpus document set as Table~\ref{tab:equilibrium}: 60 documents for WikiText-2, 100 each for Code and OpenWebText. An earlier version truncated Code and OpenWebText to 60 documents for cross-corpus alignment; this shift changes OpenWebText's original $U^*$ by $0.042$, a substantial fraction of its bootstrap interval width, and the truncation had no principled justification, so we treat it as a methodological error. Main-text numbers use the full 100-document sets.

\subsection{Seed-level inference at $n=3$}
\label{app:seedinference}

Several comparisons rest on three seeds per condition (the
$p_1$-substitution ordering of Section~\ref{sec:p1substitution},
the validation-loss comparison of Appendix~\ref{app:valcost}, and the
probe comparison of Section~\ref{sec:probing}). With three seeds,
the strongest \emph{achievable} evidence for a consistent sign is a
sign test over $n=3$: the minimum two-sided $p$-value is
$2\cdot(1/2)^3=0.25$ and the minimum one-sided $p$-value $0.125$. No
three-seed claim can reach conventional thresholds, so we report all
such comparisons as \emph{directional, seed-level} evidence rather
than hypothesis tests: ``consistent across seeds'' establishes only
that no seed contradicts the sign, the strongest statement three seeds
support.

\subsection{Layer-wise probing of paragraph-boundary decodability}
\label{sec:probing}

As a complementary, purely representational check (probing rather than
intervention), we ask at which layer, if any, a linear probe on the
residual stream predicts whether the current and previous token lie in
the same paragraph---a binary, relational target rather than $p_1$'s
absolute value. For \textbf{flat}, \textbf{rand\_axial}, and
\textbf{hrope\_axial} (three corpora, three seeds), we extract the
residual stream after each of the 8 layers and train a per-layer
logistic-regression probe, splitting train/test by document (no document
appears in both) to prevent the probe exploiting document-level surface
features. The positive class (paragraph boundary) is rare
($0.7\%$--$1.8\%$ of pairs across corpora), so we report balanced
accuracy, macro-F1, and AUROC rather than raw accuracy;
Table~\ref{tab:probe} gives balanced accuracy.

\begin{table*}[!htbp]
\centering
\caption{Balanced accuracy of the same-paragraph probe at layers 0, 4, and 7 (mean over three seeds).}
\label{tab:probe}
\begin{tabular}{lccccccccc}
\toprule
& \multicolumn{3}{c}{WikiText-2} & \multicolumn{3}{c}{Code} & \multicolumn{3}{c}{OpenWebText} \\
Model & L0 & L4 & L7 & L0 & L4 & L7 & L0 & L4 & L7 \\
\midrule
flat & 0.60 & 0.61 & 0.61 & 0.60 & 0.69 & 0.68 & 0.60 & 0.68 & 0.68 \\
rand\_axial & 0.59 & 0.75 & 0.82 & 0.65 & 0.77 & 0.80 & 0.64 & 0.89 & 0.91 \\
hrope\_axial & 0.60 & 0.78 & 0.85 & 0.65 & 0.93 & 0.95 & 0.64 & 0.88 & 0.89 \\
\bottomrule
\end{tabular}
\end{table*}

Two patterns emerge. First, \textbf{flat}'s probe accuracy rises only
modestly with depth (to $0.61$--$0.68$) and plateaus below either
axial variant in every corpus: some paragraph-boundary information is
weakly present without a dedicated $p_1$ channel, but a channel
(real or random) makes it far more linearly accessible. Second,
\textbf{hrope\_axial}'s advantage over \textbf{rand\_axial} is real but
corpus-dependent, which we checked at the individual-seed level rather
than means alone: in \textbf{Code}, hrope\_axial exceeds
rand\_axial at every seed from layer 2 on (directional at $n=3$;
Appendix~\ref{app:seedinference}), growing to $+0.11$ to
$+0.21$ balanced accuracy by layer 7---a robust representational gap.
In \textbf{WikiText-2}, the gap is positive in most mid-layers but its
sign is \emph{not} consistent across seeds by the final layer (one seed
$+0.098$, one $+0.001$, one $-0.020$), so at that layer the
apparent advantage is not distinguishable from seed noise. In
\textbf{OpenWebText}, the sign flips between seeds at nearly every
layer, leaving the two not reliably distinguishable. This mirrors the
same Code~$>$~WikiText-2~$>$~OpenWebText gradient
found in Section~\ref{sec:p1substitution} via causal intervention rather
than probing---two methods agreeing that OpenWebText is the
corpus where true and architecturally matched but uninformative
positional signals are hardest to tell apart.

\subsection{Paragraph-level compression regression}
\label{app:parareg}
\paragraph{Common-slope check.} The right block of
Table~\ref{tab:parareg} compares slopes between each pair of corpora.
Because the three per-corpus regressions are fitted on disjoint data,
their slope estimates are independent, and the difference of two raw
slopes has standard error $\sqrt{\mathrm{se}_1^2+\mathrm{se}_2^2}$; we
report $|z|$ and use a Bonferroni threshold over the twelve comparisons
(four predictors $\times$ three pairs), $|z|>2.87$ for family-wise
$0.05$.

This appendix gives the estimator, predictor set, data, protocol,
numbers, and robustness for the paragraph-level regression summarized in
Section~\ref{sec:parareg}.

\paragraph{Per-paragraph compression depth.}
We fix the shallowest cross-paragraph displacement, $\deltap{=}1$, and
define each boundary pair's contribution as its log-attention relative to
a matched within-paragraph baseline,
\begin{equation}
D_{{\rm eff},1}(i,j)
=
\log\bar A(i,j)
-
\mathrm{base}\bigl(d(i,j), b(i)\bigr),
\label{eq:upair}
\end{equation}
where $\mathrm{base}(d,b)$ is the mean $\log\bar A$ over
same-paragraph ($\deltap{=}0$) pairs pooled across the three seeds at
the same exact token distance $d$ and the same query context band
$b(i)=\lfloor\log_2(i+2)\rfloor$, with $i$ the query's absolute
position.
The second matching axis is necessary: within fixed $d$, $\log\bar A$
also decays with $i$ simply because the softmax context widens---over
$\deltap{=}0$ pairs at $d\in[120,230]$ the mean $\log\bar A$ falls from
$-4.86$ at band $6$ to $-6.37$ at band $9$---and a $d$-only baseline
would mechanically misattribute that decay to paragraph position.
Averaging $D_{{\rm eff},1}$ over the cross-paragraph pairs in which a
paragraph $p$ appears as an endpoint (a pair between adjacent
paragraphs is counted once per endpoint),
\begin{equation}
U_p
=
\frac{1}{|N(p)|}\sum_{(i,j)\in N(p)} D_{{\rm eff},1}(i,j),
\label{eq:up}
\end{equation}
with $N(p)$ the $\deltap{=}1$ pairs incident to $p$.
Negative $U_p$ is compression: $U_p$ carries the same sign convention as
the corpus-level $U^*$ of Section~\ref{sec:exactd} and
Appendix~\ref{app:binsweep}, so the depth ordering is read off the
\emph{magnitude} of a negative $U_p$ (equivalently, we may report $-U_p$ as a
positive depth). All $U_p$ values in Section~\ref{sec:parareg} and in
this appendix are reported in this signed convention.
$U_p$ is the paragraph-level analogue of the displacement-one point of
the corpus curve $U(\deltap{=}1)$, not the vertex depth $U^*$
(Section~\ref{sec:exactd}); only Code's well sits near $\deltap{=}1$.

\paragraph{Predictors.}
For each paragraph we record its token length; its type--token ratio, a
lexical-diversity proxy; its embedding coherence with the adjacent
paragraphs of the same document (mean all-MiniLM-L6-v2 cosine,
corpus-derived and seed-invariant, the same embeddings as
Section~\ref{sec:coherence-candidates}); and its fractional position in
the document, $pos/N_p$, included as a control because of the band
effect above.
Within each corpus we fit a weighted least squares model of $U_p$ on
these four predictors, weighted by each paragraph's number of
underlying pairs (1{,}167 Code, 438 WikiText-2, 1{,}285 OpenWebText
paragraphs).
Table~\ref{tab:parareg} reports standardized coefficients
$\beta_{\rm std}=b\cdot{\rm sd}_{\rm corpus}(x)$, so effect sizes are
comparable across predictors and corpora; the weights, standard errors,
and heterogeneity tests are given below.

\paragraph{Data, pooling, and estimation.}
All estimates use the same per-corpus document set as
Table~\ref{tab:equilibrium} (60 documents for WikiText-2, 100 for Code,
100 for OpenWebText; Appendix~\ref{app:binsweep}), running fresh forward
passes of the pooled three-seed \textbf{hrope\_axial} checkpoints through
the 1024-token training window.
Same-paragraph ($\deltap=0$) baseline cells and $\deltap=1$ pair cells are
accumulated per $(d,\mathrm{band})$, pooled across seeds; a baseline cell
is used only if it contains at least $400$ token pairs across the three
seeds ($899$, $959$, and $848$ usable $(d,\mathrm{band})$ cells for Code,
WikiText-2, and OpenWebText), and a paragraph is dropped if its document
was truncated by the window, if it lacks a $\deltap=1$ neighbor, or if
every incident pair fails the baseline gate.
This leaves 1{,}167 Code, 438 WikiText-2, and 1{,}285 OpenWebText
paragraphs.
Token distance is capped at $d\le256$ for the paragraph probe: the
model is trained in 1024-token blocks, but at the finer
$(d,\mathrm{band})$ granularity the same-paragraph baseline cells thin
out at larger $d$; the corpus-level protocol
(Section~\ref{sec:exactd}) instead pools over $d$ and reaches
$d\le1024$.
Within each corpus, $U_p$ is regressed on raw length,
type--token ratio, embedding coherence, and position by weighted least
squares with weights $|N(p)|$, the number of $\deltap=1$ pairs incident
to the paragraph (means $22{,}162$ for Code, $75{,}802$ for
WikiText-2, $19{,}802$ for OpenWebText).
Standard errors are the classical (homoskedastic) WLS standard errors,
with the residual variance estimated from the weighted residuals, and
significance is from $t$ on $N-5$ degrees of freedom.

\paragraph{Why the query-context band is a control, not a predictor.}
Section~\ref{sec:parareg} gives the mechanism and the band-decay
diagnostic; here we note the consequence for the regression.
A baseline matching only $d$ treats the absolute-position decay, which
is independent of paragraph structure, as cross-paragraph signal: in
pilot runs on a subset of WikiText-2 this left an overwhelming,
spurious position term (univariate $r\approx-0.96$, $R^2\approx0.9$).
Matching the baseline on $(d,b)$ removes most of it, leaving the
position control with standardized coefficients $-0.012$, $-0.123$,
$-0.099$ in Code, WikiText-2, OpenWebText, which we report as a
residual control rather than an effect.
The residual is corpus-dependent (pairwise slope differences up to $|z|=17.9$; Table~\ref{tab:parareg}), consistent
with the artifact's magnitude tracking how attention is distributed
over the context in each corpus.

\paragraph{Implied raw slopes.}
Per-corpus raw slopes (Table~\ref{tab:parareg-raw}; classical WLS
standard errors in parentheses):

\begin{table*}[!htbp]
\centering
\small
\renewcommand{\arraystretch}{0.9}
\caption{Per-corpus raw slopes of the paragraph-level regression of $U_p$ (Eq.~\ref{eq:up}); standard errors in parentheses. Signs follow that equation: negative $U_p$ is compression, so a positive coherence slope denotes a \emph{smaller} depth magnitude, not deeper compression. Standardized equivalents (slope $\times$ per-corpus predictor SD; the values reported in Table~\ref{tab:parareg}) are, for Code/WikiText-2/OpenWebText: length $-0.044/-0.016/+0.005$; type--token $+0.052/-0.007/-0.004$; coherence $+0.015/+0.013/+0.017$; position $-0.012/-0.123/-0.099$.}
\label{tab:parareg-raw}
\begin{tabular}{lrrrr}
\toprule
& Length & Type--token & Coherence & Position\\
\midrule
Code & $-0.00053\,(0.00005)$ & $+0.239\,(0.031)$ & $+0.079\,(0.025)$ & $-0.042\,(0.018)$\\
WikiText-2 & $-0.00025\,(0.00008)$ & $-0.068\,(0.069)$ & $+0.078\,(0.025)$ & $-0.481\,(0.016)$\\
OpenWebText & $+0.00011\,(0.00007)$ & $-0.033\,(0.040)$ & $+0.109\,(0.021)$ & $-0.353\,(0.012)$\\
\bottomrule
\end{tabular}
\end{table*}

Per-corpus predictor standard deviations convert these to the
standardized entries of Table~\ref{tab:parareg}: e.g.\ Code's length
slope $-0.00053\times82.5=-0.044$.
The three coherence slopes ($0.079$, $0.078$, $0.109$) have heavily
overlapping intervals, and no pair differs ($|z|\le0.95$;
Table~\ref{tab:parareg}).

\paragraph{Coherence level versus slope.}
The universal relation in Table~\ref{tab:parareg} is the coherence
\emph{slope}; its corpus \emph{level} is not aligned with depth
(per-paragraph mean adjacent-embedding cosine is $0.514$ for
WikiText-2, higher than Code's $0.400$ and OpenWebText's $0.398$).
The corpus-level ``aggregate coherence'' of
Table~\ref{tab:allcandidates} is a different summary---the adjacency
cosine minus a random-pair baseline, a level rather than a
within-corpus slope---which is why the two statistics can disagree
without contradiction: the aggregate is corpus-conditioned by the
embedding recipe of Section~\ref{sec:coherence-candidates}, the
paragraph slope by the regression here, and we claim only the latter as
a common relation.
The two also disagree in \emph{sign}: the corpus-level aggregate ranks
with depth (Table~\ref{tab:allcandidates}), whereas the conditional
paragraph-level slope is positive in all three corpora, which under the
sign convention of Eq.~\ref{eq:up} means a \emph{smaller} depth magnitude
for more coherent paragraphs.
We report this opposition rather than reconcile it: the conditional
slope absorbs length, diversity, and position effects, so a positive
between-corpus coherence--depth association need not induce a positive
within-corpus partial slope.

\subsection{A single-example demonstration}
\label{app:demo}

Figure~\ref{fig:love-demo-full} gives the full illustrative single-example
demonstration referenced in Section~\ref{sec:intro} (Figure~\ref{fig:love-demo}) and
Section~\ref{sec:layer1}: attention maps for the original (Cond.\ A) and
manipulated (Cond.\ B) conditions on the token sequence \textit{This is
what I love. You are the reason.}, their difference, and the
corresponding $D_{\mathrm{eff}}$ values. Only $p_1$ changes; the token
sequence is identical across conditions. This is an $n=1$
illustration, not a statistical result.

\begin{figure*}[!htbp]
\centering
\includegraphics[width=\textwidth]{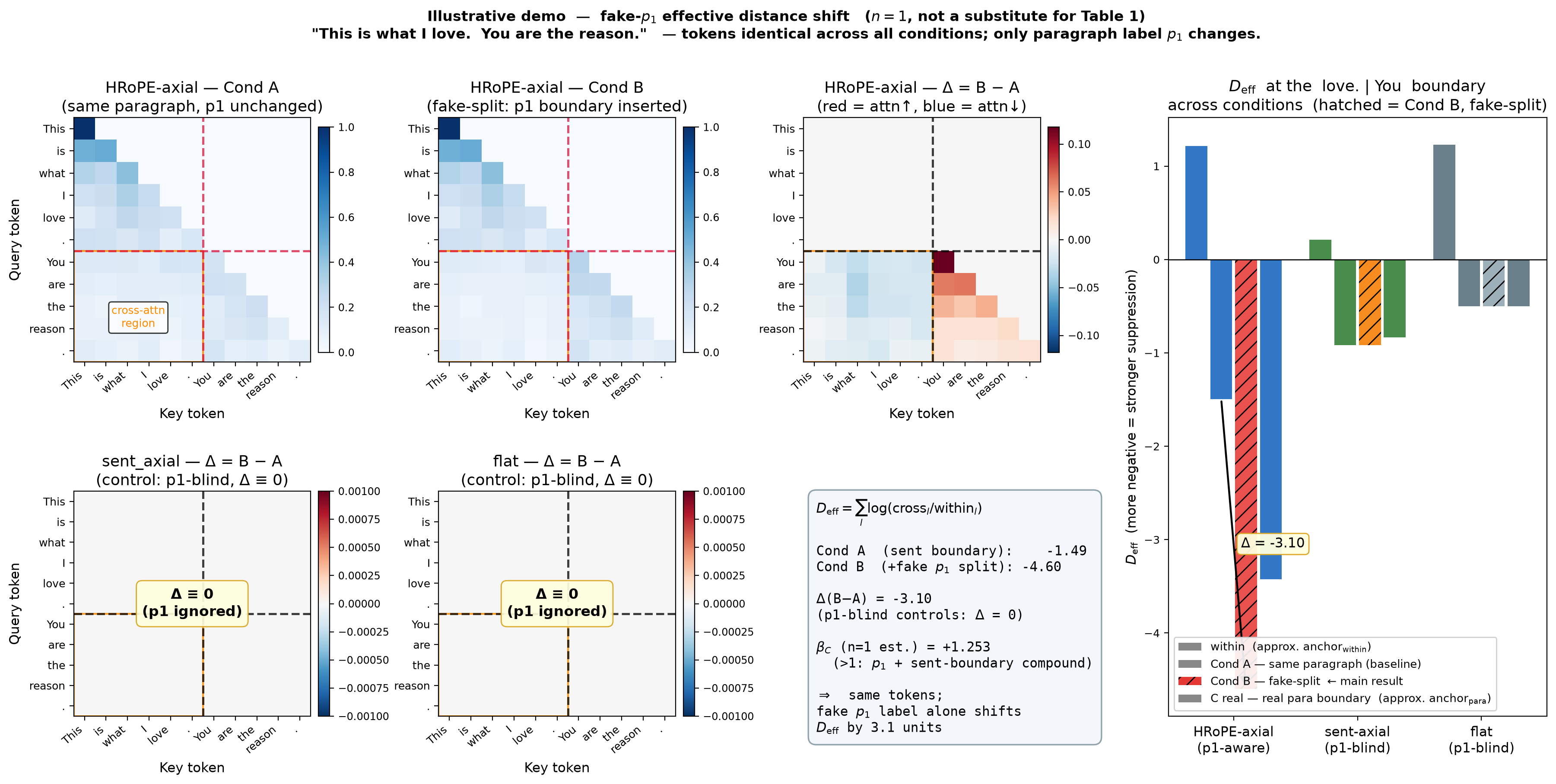}
\caption{Full single-example demonstration. \emph{Top}: \textbf{hrope\_axial}
attention maps for Cond A, Cond B, and their difference. \emph{Bottom}: the
same difference for the $p_1$-blind controls \textbf{sent\_axial} and
\textbf{flat}, identically zero. \emph{Right}: $\Deff$ at the ``love.~$|$~You''
boundary across conditions and models.}
\label{fig:love-demo-full}
\end{figure*}

\subsection{The compression response curves}
\label{app:wells}

Figure~\ref{fig:potential-wells} shows the exact-distance response curve
$U(\deltap)$ of Section~\ref{sec:exactd} for all three corpora, together
with the fitted degree-2 curve whose vertex defines the compression depth
$U^*$ of Table~\ref{tab:equilibrium}.

\begin{table*}[!htbp]
\centering
\small
\renewcommand{\arraystretch}{0.9}
\caption{Compression depth $U^*$ and, where identifiable, its
displacement $\deltastar$, under the exact-distance estimator
($n\ge2{,}000$ per cell, three seeds pooled per corpus, paired
document-cluster bootstrap, 2000 draws; Section~\ref{sec:exactd};
Appendix~\ref{app:binsweep}).}
\label{tab:equilibrium}
\begin{tabular}{lrrl}
\toprule
Corpus
&
Depth $U^*$ (95\% CI)
&
$\deltastar$
&
Position note
\\
\midrule
Code
& $-0.772$ ($[-0.812,-0.742]$)
& $\approx5.5$ (deg2) / $\approx3.4$ (deg3)
& identifiable, fit-sensitive
\\
WikiText-2
& $-0.481$ ($[-0.532,-0.422]$)
& n/a (see text)
& \textbf{not identifiable}
\\
OpenWebText
& $-0.305$ ($[-0.367,-0.249]$)
& $\approx2.6$
& identifiable
\\
\bottomrule
\end{tabular}
\end{table*}

\subsection{The two counterfactual interventions}
\label{app:interventions}

Figure~\ref{fig:interventions} illustrates the two counterfactual
interventions of Section~\ref{sec:framework} on a three-sentence
document: the original segmentation, an artificial merge in which two
adjacent paragraphs receive the same paragraph coordinate, and an
artificial split in which a boundary is introduced within a paragraph.
Only $p_1$ changes in either intervention; the token sequence is
identical across all three conditions.

\begin{figure*}[!htbp]
\centering
\begin{tikzpicture}[
    cell/.style={draw, thick, minimum width=2.8cm, minimum height=0.7cm, align=center, font=\small},
    grpbox/.style={draw, thick, rounded corners=2pt, inner sep=4pt, dashed},
    arrow/.style={<->, thick, >=stealth, gray!50}
]

\node[font=\small\bfseries] at (0, 3.0) {Original};

\node[cell, fill=blue!10] (A1) at (0, 2.0) {Sentence A};
\node[cell, fill=blue!10] (B1) at (0, 1.0) {Sentence B};
\node[cell, fill=blue!10] (C1) at (0, 0.0) {Sentence C};

\node[font=\small\bfseries] at (5.0, 3.0) {Artificial-Merge};

\node[cell, fill=blue!20] (A2) at (5.0, 2.0) {Sentence A};
\node[cell, fill=blue!20] (B2) at (5.0, 1.0) {Sentence B};
\node[cell, fill=blue!10] (C2) at (5.0, 0.0) {Sentence C};

\node[font=\small, blue!60] at (5.0, 2.8) {\emph{merged}};

\draw[grpbox, blue!50] (3.6, 2.5) rectangle (6.4, 0.5);

\node[font=\small\bfseries] at (10.0, 3.0) {Artificial-Split};

\node[cell, fill=blue!10] (A3) at (10.0, 2.0) {Sentence A};
\node[cell, fill=blue!20] (B3a) at (8.8, 1.0) {B (part 1)};
\node[cell, fill=blue!20] (B3b) at (11.2, 1.0) {B (part 2)};
\node[cell, fill=blue!10] (C3) at (10.0, 0.0) {Sentence C};

\node[font=\small, blue!60] at (10.0, 2.8) {\emph{split}};

\draw[grpbox, blue!50] (8.4, 1.5) rectangle (11.6, 0.5);

\draw[arrow] (1.6, 1.0) -- (3.2, 1.0);
\draw[arrow] (6.6, 1.0) -- (8.2, 1.0);

\node[font=\footnotesize, align=center, text width=14cm] at (5.0, -1.2) {
Only the paragraph coordinate $p_1$ changes; the token sequence is identical across all three conditions.
};

\end{tikzpicture}
\caption{Illustration of the two counterfactual interventions. \emph{Left}: original paragraph segmentation. \emph{Center}: artificial-merge, where two adjacent paragraphs are assigned the same paragraph coordinate. \emph{Right}: artificial-split, where an artificial boundary is introduced within a paragraph.}
\label{fig:interventions}
\end{figure*}

\subsection{Extended interpretation and scope}
\label{app:extgeom}

\paragraph{What ``geometry'' means here, and the human-reading intuition.}
If the effective relationship between two tokens depends not only on their reading-order distance $|i-j|$ but also on hierarchical coordinates, $d_{\rm eff}(i,j)=F\big(|i-j|,|p_1(i)-p_1(j)|,|p_2(i)-p_2(j)|,\ldots\big)$, then a conventional positional encoding assumes the first argument alone suffices, and our results suggest this can be insufficient. ``Geometry'' here is descriptive: the strongest justified statement is a \emph{structured, coordinate-dependent response field}, with ``potential-well'' terminology implying no complete geometric theory and no simulation of human reading (Appendices~\ref{app:extgeom} and~\ref{app:extlimits}).

\paragraph{Why ``geometry'' is descriptive only.} A first-order
intervention cannot identify a second-order metric tensor, so the
non-monotonic response documented in Section~\ref{sec:exactd} does not by
itself establish a Riemannian metric or physical potential; the strongest
justified statement is that attention exhibits a structured,
coordinate-dependent response field. The same caveat applies to the
human-reading intuition: we do not claim that the phenomenon simulates
human reading, only that a mathematically analogous distinction between
linear and structural distance appears inside the model, and that this
distinction has a causal, corpus-dependent depth.

\subsection{A superposition of monotone processes cannot produce the observed turnover}
\label{app:two-process}

A natural mechanistic hypothesis for the non-monotonic response is an
additive combination of two independent monotonic processes in
$\deltap$---an attraction decaying with displacement plus a suppression
growing with it. Any such model is itself monotonically non-decreasing
and so cannot produce an interior minimum
(Proposition~\ref{prop:twoprocess}),
which rules out this simplest additive account for the Code and
OpenWebText turnovers while leaving the interaction that does produce
them open.
Here is the construction and the proof.

A natural mechanistic hypothesis for the non-monotonic response is an additive combination of two independent monotonic processes: a coherence-driven attraction strongest at short paragraph displacement and decaying with $\deltap$, plus an interference-driven suppression growing with $\deltap$. Formally, consider the model
\begin{equation}
U(\deltap) = -a\,e^{-\alpha\deltap} + c(\deltap),
\label{eq:twoprocess}
\end{equation}
with $a,\alpha>0$ and $c(\deltap)$ monotonically non-decreasing (e.g.\ linear, $c=c\deltap$, or saturating, $c(\deltap)=b(1-e^{-\beta\deltap})$ with $b,\beta>0$). This covers any model in which two processes, each individually monotonic in $\deltap$, are added to produce $U$---including two attractive terms with different decay rates, not only an attraction-plus-repulsion pair. We rule out this additive class, a natural but not the only family compatible with the observed response: softmax-competitive (non-additive) and non-monotonic alternatives are untested here. Two sources could instead compete through a shared softmax normalization, or a source could be anchored at a paragraph offset other than the query's own, making its distance a non-monotonic function of $\deltap$.

\begin{proposition}
\label{prop:twoprocess}
Any model of the form \eqref{eq:twoprocess}, with a monotonically
decaying attractive term and a monotonically non-decreasing suppressive
term, is itself monotonically non-decreasing in $\deltap$ and therefore
cannot produce an interior minimum.
\end{proposition}
\begin{proof}
The discrete first difference is
\begin{equation}
\begin{split}
U(\deltap+1)-U(\deltap) &= a\,e^{-\alpha\deltap}\left(1-e^{-\alpha}\right) \\
&\quad + \left[c(\deltap+1)-c(\deltap)\right].
\end{split}
\end{equation}
The first term is strictly positive for all $\deltap$, since
$a,\alpha>0$ implies $e^{-\alpha\deltap}>0$ and $1-e^{-\alpha}>0$. The
second term is non-negative by the assumed monotonicity of $c$. Their
sum is therefore strictly positive for every $\deltap$, so $U$ is
strictly increasing and admits no interior minimum.
\end{proof}

This fact directly contradicts the observed response in Code and OpenWebText, the two corpora with an identifiable interior minimum: in both, attraction genuinely \emph{strengthens} out to $\deltap\approx2.6$--$5.5$ before reversing (the location is fit-sensitive for Code, Section~\ref{sec:exactd}, but the turnover holds under both degree-2 and degree-3 fits)---a shape no monotone-decaying-plus-monotone-growing superposition can produce. WikiText-2 is excluded from, not contradicted by, this argument: its depth is stable, but the location at which it is reached is structurally non-identifiable (Section~\ref{sec:exactd}), so we cannot say whether its curve turns over in the sampled range; its depth is independently established by the rand\_axial comparison above.

We read this as a sanity check, not a mechanistic elimination: it rules out the simplest additive account, locating the turnover in some interaction without identifying which. The leading untested alternative is a softmax-competitive model, in which two logit sources are normalized jointly rather than added---a source dominating at small $\deltap$ whose relative weight is diluted as a second grows could plausibly produce the same strengthen-then-weaken shape. We do not fit such a model and flag it as the main open mechanistic question.

\section{Extended Limitations}
\label{app:extlimits}

Section~\ref{sec:limits} lists the limitations in brief; this appendix
gives the detailed statement of each.

\paragraph{No corpus-only predictor explains the corpus-dependent depth.} None of the three constructs we tested (lexical persistence, paragraph length, embedding-based coherence; eight quantities total) reproduces the cross-corpus ordering of $U^*$, and none predicts the second, separate fact that the real-versus-random gap itself is resolvable in two corpora but not the third (Section~\ref{sec:randcontrol}). This rules out the specific accounts tested, not that some other corpus-level property, or the architecture and training dynamics themselves, determine depth; we regard both as open.

\paragraph{The paragraph-level probe is within-corpus and correlational.} Section~\ref{sec:parareg} re-asks the comparison of Section~\ref{sec:alt-scales} one level down, at the cost of a narrower scope: it estimates the per-paragraph slope of $U_p$ on paragraph properties with real within-corpus power, and finds it common for embedding coherence but corpus-specific for length, diversity, and position; but the slope analysis is correlational and uses $\deltap{=}1$ only, so it neither identifies a mechanism nor supplies the missing corpus-level scale. Its coherence slope also runs opposite in sign to the corpus-level coherence--depth ranking, and what the probe establishes is a conditional within-corpus relation rather than a bivariate law.

\paragraph{Scope and units.} Three corpora limit any cross-corpus claim to the descriptive, and paragraphs are treated as the relevant unit under a corpus-specific definition (Section~\ref{sec:framework}); other units may behave differently. ``Geometry'' is used descriptively, not as a claim about a complete metric or mechanism. As a sanity check, Section~\ref{app:two-process} rules out only the additive competing-process family; softmax-competitive and non-monotonic alternatives remain untested. The exact-distance protocol ($n\ge2{,}000$, $d\le1{,}024$) is conservative by construction, which is precisely why $\deltastar$ becomes non-identifiable for WikiText-2 rather than merely uncertain (Appendix~\ref{app:pooling}, \ref{app:binsweep}).

\paragraph{Cross-corpus ordering is descriptive, not inferential.} The document-level bootstrap quantifies \emph{within-corpus}, not cross-corpus, sampling uncertainty, and the three corpora are fixed by choice rather than drawn from a common population. Statements about the ordering therefore establish that the corpora differ in depth, not that the specific order generalizes. We accordingly avoid ``statistically significant'' language for cross-corpus comparisons and present each pairwise claim with its within-corpus confidence interval.

\paragraph{Scale and downstream relevance.} All experiments use 8-layer, $d_{\rm model}=512$ models trained for 5000 steps, chosen to make the multi-corpus, multi-seed, causal-intervention protocol tractable, not for production scale. \textbf{hrope\_axial} never costs more validation loss than \textbf{rand\_axial}, and costs less in every matched seed and corpus (directional at $n=3$; Appendix~\ref{app:seedinference}); its cost relative to \textbf{flat} is small, corpus-dependent, and indistinguishable from seed noise in Code (Appendix~\ref{app:valcost}). We report no downstream-task metric and do not evaluate at production scale; whether the phenomenon persists at larger scale, and whether it affects downstream performance, remain open.

\section{Training and Implementation Details}
\label{app:training}

All corpora are tokenized with the GPT-2 byte-pair encoding
(\texttt{tiktoken}, \texttt{gpt2} encoding, vocabulary size 50257).

\subsection{Corpus-specific boundary extraction}

\paragraph{WikiText-2.} Document boundaries are identified by
top-level heading lines (single-equals headings, regex
\verb|^\s*=\s[^=].*[^=]\s=\s*$|); heading lines themselves are discarded.
Paragraphs are split on blank lines (\verb|\n\s*\n|). Sentences are split
on sentence-final punctuation followed by whitespace and an uppercase
letter or quotation mark (\verb|(?<=[.!?])\s+(?=[A-Z"'])|). Paragraphs
shorter than 20 characters are dropped. WikiText-2's \verb|@-@|-style
escaped punctuation is normalized back to standard punctuation before
segmentation.

\paragraph{OpenWebText.} We follow the standard OpenWebText
construction pipeline: web pages are extracted with \texttt{newspaper},
non-English content is filtered with FastText, near-duplicate
documents are removed via 5-gram LSH (similarity threshold $>0.5$),
and documents shorter than 128 tokens are discarded. This yields
$8{,}013{,}769$ documents total from the \texttt{Skylion007/openwebtext}
release. Unlike WikiText-2, we do not apply an additional
per-paragraph character-length filter or a per-document
paragraph-count exclusion.

\paragraph{Code.} For the Python source-code corpus, the hierarchy is
file $\to$ block ($p_1$) $\to$ line ($p_2$) $\to$ token ($p_3$). Blocks
are identified by splitting on blank lines, merging consecutive non-blank
lines into a single block; lines are obtained by splitting each block on
newlines, discarding empty lines. Files with fewer than 2 blocks are
excluded, and $p_1$ values are truncated at the 95th percentile of
blocks-per-file (131 blocks) to exclude atypically long files from base
calibration and from fake-$p_1$ manipulation sampling, consistent with the
non-wrapping base calibration described below.

\subsection{Training configuration}

\begin{table*}[!htbp]
\centering
\caption{Training hyperparameters, held fixed across all model variants (\textbf{flat}, \textbf{sent\_axial}, \textbf{hrope\_axial}, \textbf{rand\_axial}, and the mirror control \textbf{period\_axial}, Section~\ref{sec:hrope}) and corpora.}
\label{tab:hyper}
\begin{tabular}{lp{0.55\textwidth}}
\toprule
Parameter & Value \\
\midrule
$d_{\rm model}$ & 512 \\
Layers & 8 \\
Attention heads & 8 \\
Block size (context length) & 1024 \\
Batch size & 16 \\
Training steps & 5000 \\
Learning rate & $3\times10^{-4}$ \\
Optimizer & AdamW (weight decay $0.01$) \\
LR schedule & linear warmup (100 steps) + cosine decay (min lr $=0.01\times$ base lr) \\
Gradient clipping & max norm $1.0$ \\
Dropout & 0.1 \\
Early stopping & configured (patience 10 evaluations), never triggered; all logged runs reached 5000 steps \\
Seeds & 0, 1, 2 for all corpora (WikiText-2, OpenWebText, Code) \\
\bottomrule
\end{tabular}
\end{table*}

\paragraph{Data volume.} Code: approximately 49{,}000 training files (80\%
of the collected corpus), 2000 validation files. OpenWebText: following
nanoGPT's \texttt{prepare.py} split of the
\texttt{Skylion007/openwebtext} release (which provides only a train
split), we hold out $0.05\%$ of documents (\texttt{test\_size}$=0.0005$,
seed $2357$) for validation, giving $8{,}009{,}762$ training and
$4{,}007$ validation documents ($\approx$9B and $\approx$4.4M tokens,
respectively; $\approx$17GB training / $\approx$8.5MB validation).

\paragraph{Per-axis channel allocation and base calibration.}
Section~\ref{sec:hrope} assigns each of the three coordinates its own
rotary channel group. For a group with $D_g$ channel pairs and a
99th-percentile position range $M_g$ measured over the corpus we set
\[
\mathrm{base}_g = \left(\frac{S\cdot M_g}{2\pi}\right)^{D_g/(D_g-1)}
\]
with safety margin $S>1$, so the lowest-frequency channel stays
monotonic and non-wrapping over each axis's typical range. The
calibration uses corpus statistics only and adds no learned
parameters. The per-axis channel allocation and the resulting base
values are matched by construction across the axial variants
(\textbf{sent\_axial}, \textbf{hrope\_axial}, \textbf{rand\_axial},
\textbf{period\_axial}) and across the three seeds, so that the
variants differ only in what occupies the $p_1$ channel.

\paragraph{The \textbf{period\_axial} mirror control.}
\textbf{period\_axial} is identical to \textbf{hrope\_axial} and
\textbf{rand\_axial} in architecture, parameter count, channel
allocation, schedule, and seeds, but its paragraph coordinate is the
mechanical grid $p_1:=\lfloor t/L\rfloor$ at the same density as the
true segmentation ($L=132/102/64$ tokens for
WikiText-2/Code/OpenWebText) rather than the true paragraph position.
Its role (Section~\ref{sec:p1substitution}) is to mirror the
$p_1$-substitution test: substituting the true paragraph coordinate into
this checkpoint shows whether a model trained on a mechanical coordinate
becomes sensitive to genuine paragraph boundaries it never saw. Because this $p_1$ is
almost deterministic in token distance within a paragraph-length
window, \textbf{period\_axial} cannot be analyzed by the
at-matched-distance controlled-contrast protocol
(Section~\ref{sec:exactd}; Appendix~\ref{app:pooling}) and therefore enters no depth ($U^*$) comparison; its only use is the
loss-based substitution check of Section~\ref{sec:p1substitution}. The
depth comparisons and the validation-loss comparison of
Appendix~\ref{app:valcost} use only the four primary variants.

\paragraph{The \textbf{rand\_axial} control.}
\textbf{rand\_axial} is identical to \textbf{hrope\_axial} in
architecture, parameter count, channel allocation, base calibration,
schedule, and seeds; only the content of its $p_1$ channel differs. That
channel carries a random label $\rho(i)$ matched to the boundary density
of the true segmentation and resampled independently at every training
step, so it is never fixed for a document and no position can be
memorized as a boundary (Section~\ref{sec:randcontrol}).

\section{Language-Modeling Cost of the $p_1$ Channel}
\label{app:valcost}
\paragraph{What the substitution results do and do not establish.} The
protocol does not show that an ordinary Transformer without a dedicated
$p_1$ channel implicitly represents paragraph structure---that is
outside our scope; it shows that, once such a
channel exists, the trained weights are sensitive to what occupies it,
not merely to its presence. In the seed-matched validation-loss
comparison of the four primary variants (\textbf{period\_axial} is not
included), the causally active $p_1$ channel never
costs more validation loss than \textbf{rand\_axial}, while its cost
relative to \textbf{flat} is small and corpus-dependent, and in Code
indistinguishable from seed noise. A complementary logistic-regression
probe on the residual stream of \textbf{flat}, \textbf{rand\_axial}, and
\textbf{hrope\_axial} yields the same Code $>$ WikiText-2 $>$ OpenWebText
gradient (Appendix~\ref{sec:probing}), so the two methods agree on where
true and architecturally matched but uninformative positional signals are
hardest to distinguish.

The reverse substitution of Section~\ref{sec:p1substitution} (true
$p_1$ fed to the \textbf{period\_axial} checkpoint, $+0.015$ to
$+0.040$ nats across the three corpora) uses the inference-only protocol
of Table~\ref{tab:p1sub}, not the training-time comparison below, and
is therefore not part of Table~\ref{tab:valloss}.

This appendix supplements Section~\ref{sec:layer1}'s claim that
hierarchical position is a causal factor in attention with a
seed-matched validation-loss comparison (three seeds, no formal
bootstrap CI, directional at $n=3$;
Appendix~\ref{app:seedinference}). Table~\ref{tab:valloss} reports
validation loss (mean $\pm$ std over the three seeds used
throughout) for the four positional variants under the identical
configuration of Table~\ref{tab:hyper}.

\begin{table*}[!htbp]
\centering
\caption{
Validation loss (mean $\pm$ std over three seeds) by positional
variant, using the training configuration of
Table~\ref{tab:hyper}. Lower is better.
}
\label{tab:valloss}
\begin{tabular}{lcccc}
\toprule
Corpus & flat & sent\_axial & hrope\_axial & rand\_axial \\
\midrule
Code
& $2.385\pm0.050$
& $2.504\pm0.016$
& $2.383\pm0.036$
& $2.507\pm0.036$
\\
WikiText-2
& $5.383\pm0.014$
& $5.480\pm0.015$
& $5.411\pm0.007$
& $5.437\pm0.002$
\\
OpenWebText
& $5.593\pm0.013$
& $5.718\pm0.005$
& $5.622\pm0.004$
& $5.669\pm0.006$
\\
\bottomrule
\end{tabular}
\end{table*}

With three seeds per variant, comparing means alone risks
mistaking sampling noise for a real effect, so we also compare
seed-matched values (seeds 0, 1, 2 shared across variants),
checking whether one variant's loss exceeds the other's for every
matched seed and whether the two variants' seed ranges are disjoint.

This comparison is consistent for
\textbf{rand\_axial} versus \textbf{hrope\_axial}: rand\_axial's loss
exceeds hrope\_axial's in all three matched seeds of every corpus, and
the seed ranges do not overlap in any corpus. At $n=3$ this is
directional, seed-level evidence (minimum one-sided
sign-test $p=0.125$; Appendix~\ref{app:seedinference}): the
architecturally identical channel carrying no true positional
information is more costly than one that does, rather than a
precisely quantified or formally significant cost.

The \textbf{hrope\_axial}--\textbf{flat} comparison is
weaker and corpus-dependent. In WikiText-2, hrope\_axial's loss
exceeds flat's for all three matched seeds and the ranges do not
overlap ($\approx0.5\%$ relative), consistent with a small, real
cost; OpenWebText shows the same sign and non-overlap ($\approx0.5\%$
relative). In Code the sign is \emph{not} consistent
(hrope\_axial higher for two seeds, lower for the third; mean
difference $\approx-0.07\%$), and hrope\_axial's seed range falls
entirely inside flat's; we read this as consistent with sampling
noise at $n=3$, not evidence of a real cost in Code.

Taken together, \textbf{hrope\_axial} never costs more than
\textbf{rand\_axial}, and costs less in every
matched seed and corpus. Whether it costs
anything over \textbf{flat} is less certain: a small, seed-consistent
$\approx0.5\%$ in WikiText-2 and OpenWebText, and no detectable
cost in Code. Either way, we find no case in which the causally active $p_1$ channel
costs more than an architecturally matched control carrying no
true positional information. This is a small-scale, $n=3$
comparison without a formal bootstrap CI (unlike the document-level
cluster bootstrap for $\Delta_B,\Delta_C$ elsewhere); we report the
seed-level sign consistency as weaker, complementary evidence, not a
quantified cost or a substitute for that protocol.

\section{Supplementary Analysis: Is Compression Uniform or Hub-Driven?}
\label{app:hubmask}
\paragraph{What the probe controls for.} Masking the same
\emph{number} of random tokens barely moves $U^*$, whereas masking the
same \emph{volume} of cross-paragraph pairs moves it substantially and
non-linearly, so the response tracks removed attention mass rather than
the identity of the masked tokens; the residual beyond this mass effect
is of similar size across corpora ($14$--$18$ percentage points of
the well). Three trained seeds pooled per corpus, inference only, no retraining;
no main-text estimate changes.

This appendix gives the full protocol, numbers, and status discussion
for the exploratory masking probe summarized in Section~\ref{sec:layer3}.

\paragraph{Protocol.}
For each corpus we use the trained \textbf{hrope\_axial} checkpoint to
compute, for every cross-paragraph token pair, its contribution to
$\Deff$ relative to the same-$d$, $\deltap=0$ baseline (the per-pair
quantity underlying $U^*$ in Table~\ref{tab:equilibrium}).
Aggregating by token (summed over all cross-paragraph
pairs in which it appears as key), we define the \emph{hub set} as
the top 5\% of tokens by aggregate contribution. All estimates use
the same per-corpus document set as Table~\ref{tab:equilibrium}
(60 documents for WikiText-2, 100 for Code, 100 for OpenWebText;
Appendix~\ref{app:binsweep}). We compare four inference-time masking
conditions on the same checkpoint, no retraining:
\textbf{Hub-masked} (cross-paragraph attention to hub tokens blocked,
$-\infty$ on the key side); \textbf{Freq-matched} (the same masking
applied to the top 5\% of tokens by raw corpus frequency, defined
independently of $\Deff$ and hence not circular with respect to the
effect under test); \textbf{Random (token-matched)} (a random
non-hub set of the same \emph{number of tokens} as the hub
set); and \textbf{Pair-matched random} (a random non-hub
token set whose number of removed cross-paragraph \emph{pairs} matches
Hub-masked exactly; $n=3$ draws for both random conditions, reported as
mean $\pm$ standard deviation). We verified, for every corpus, that the
realized hub-key fraction among pairs removed by Pair-matched random
matches the base rate, and that the number of surviving $(d,\deltap)$
cells after masking is comparable across the three non-trivially-masked
conditions, so differences between them cannot be attributed to
differing estimator stability. Contrasting the two random conditions
isolates volume from token identity: Random (token-matched)
removes far fewer cross-paragraph pairs than Hub-masked, so if masking
is driven mainly by \emph{how many pairs} are removed rather than by
\emph{which} tokens, it should barely move $U^*$
while Pair-matched random moves it substantially---the pattern we find
(Table~\ref{tab:hubmask}).

\begin{table*}[!htbp]
\centering
\caption{Compression depth $U^*$ under masking, by corpus
(exact-distance protocol; exploratory; three seeds pooled per corpus,
inference only). ``Original'' is the unmasked estimate, re-fitted within
the same script, and matches Table~\ref{tab:equilibrium}. Random conditions report
mean $\pm$ std over 3 draws; bracketed intervals are document-cluster
bootstrap 95\% CIs. $\dag$: falls back to the bootstrap mean (too few
surviving cells for a direct fit).}
\label{tab:hubmask}
\begin{tabular}{lrrr}
\toprule
Condition & WikiText-2 & OpenWebText & Code \\
\midrule
Original             & $-0.481$ & $-0.305$ & $-0.772$ \\
Hub-masked           & $-0.001$ & $+0.079$ & $-0.106$ \\
                      & \scriptsize{$[-0.044,0.039]$} & \scriptsize{$[0.031,0.104]$} & \scriptsize{$[-0.153,-0.056]$} \\
Freq-matched          & $-0.008^\dag$ & $+0.076^\dag$ & $-0.111$ \\
                      & & \scriptsize{$[0.041,0.103]$} & \scriptsize{$[-0.162,-0.063]$} \\
Random (token-matched) & $-0.468$ & $-0.296\pm0.002$ & $-0.762\pm0.001$ \\
Pair-matched random   & $-0.090\pm0.003$ & $+0.034\pm0.000$ & $-0.219\pm0.000$ \\
\bottomrule
\end{tabular}
\end{table*}

Random (token-matched) leaves $U^*$ close to Original in every corpus
($-0.468$ vs.\ $-0.481$ for WikiText-2; $-0.296$ vs.\ $-0.305$ for
OpenWebText; $-0.762$ vs.\ $-0.772$ for Code): masking the same
\emph{number} of arbitrary tokens barely changes the well. Pair-matched
random, matched instead on the \emph{volume} of cross-paragraph pairs
removed, moves $U^*$ substantially in every corpus. Volume, not token
count, drives the bulk of the effect---confirming the choice of
pair-matched masking as the baseline below.

\paragraph{A mass effect dominates, and is non-linear.}
To separate pure sample removal from any token-identity-specific
effect, we swept the fraction of cross-paragraph pairs removed by
Pair-matched random masking from small up to the fraction
removed by Hub-masked itself. In WikiText-2 and Code, $U^*$ moves
monotonically toward zero with no intermediate
discontinuity, and the marginal effect of each additional
increment of removed mass grows with the removed fraction
(a convex curve). In OpenWebText, the same accelerating, monotonic
trend continues \emph{through} zero: at the largest swept fraction,
Pair-matched random alone pushes $U^*$ from $-0.305$ to $+0.034$,
turning compression into apparent dilation before any
token-identity-specific masking is applied. A
substantial share of the change seen under Hub-masked and Freq-matched
in Table~\ref{tab:hubmask} thus reflects nothing more specific than the
volume of cross-paragraph attention mass removed, independent of which
tokens carry it---and in OpenWebText this mass effect alone
overshoots past the within-paragraph baseline into dilation.

\paragraph{The residual beyond the mass effect is corpus-dependent, but consistent in size.}
In every corpus, Hub-masked and Freq-matched agree closely,
confirming that the hub tokens are well approximated by
raw frequency: WikiText-2's are dominated by high-frequency function
words (\emph{the}, \emph{and}, \emph{to}, \emph{of}, \ldots), Code's
by whitespace, punctuation, and operator characters---simply the most
frequent tokens in source code. Expressing each masked condition as
the fraction of the original well depth it eliminates,
$1-U^*_{\rm masked}/U^*_{\rm original}$, Pair-matched random alone
accounts for $81\%$ of the change in WikiText-2, $72\%$ in Code, and
$111\%$ in OpenWebText (an overshoot past zero); Hub-masked
adds $14$--$18$ percentage points beyond
Pair-matched random in every corpus (WikiText-2 $99\%$ vs.\ $81\%$;
Code $86\%$ vs.\ $72\%$; OpenWebText $126\%$ vs.\ $111\%$). This
increment is the part of the flattening (or, for OpenWebText,
overshoot) that volume alone does not explain, and its size is
remarkably consistent across three structurally different corpora.

What differs across corpora is not the size of this hub-specific
increment but \emph{where it lands relative to zero}, set by
how much of the well the mass effect has already erased.
For WikiText-2, the mass effect alone accounts for most of
the original depth ($81\%$), leaving little room for the increment:
both Hub-masked and Freq-matched confidence intervals
include zero, and at this sample size masking specifically
the highest-contribution or highest-frequency
tokens is not distinguishable from masking an equal volume of random
tokens. For Code, the mass effect accounts for only $72\%$ of the
depth, so a similar increment ($14$ points here) lands short of zero:
both confidence intervals exclude zero on
the compression side ($[-0.153,-0.056]$ and $[-0.162,-0.063]$),
leaving a reliable compression residual that Pair-matched random's
$-0.219\pm0.000$ does not reach. For OpenWebText, the mass effect
already overshoots past zero ($111\%$), so the same increment lands
past it again, into dilation: both confidence intervals
exclude zero on the dilation side
($[0.031,0.104]$ and $[0.041,0.103]$), reliably above Pair-matched
random's own overshoot of $+0.034\pm0.000$.

In this sense Code is the most \emph{robust} corpus to
this pair of interventions, retaining a compression residual even
after mass and hub effects combine, while OpenWebText is the
most \emph{fragile}: the same combination that only
partially erodes Code's well, and exactly erases WikiText-2's, drives
OpenWebText's well through zero. This dilation is entirely an artifact
of masking: Table~\ref{tab:equilibrium} shows all three corpora
compressing, none dilating, in their \emph{unmasked} state;
the sign flip reflects how OpenWebText's compression responds to a
large volume of cross-paragraph mass removed, not the original signal.

This split is consistent with the differing syntactic role that
frequent tokens play across corpora. In Code, the most frequent tokens
--- indentation whitespace, brackets, operators --- are also the
explicit delimiters of statement and block boundaries, so masking
them removes structural information that a matched volume of random
tokens does not carry, leaving a residual that survives even the
combined mass and hub effects. In WikiText-2 and OpenWebText, frequent
function words are lexically ubiquitous but do not themselves
demarcate paragraph boundaries; masking them behaves, within the
precision of this test, like masking an equivalent volume of any other
tokens, and the consequence of that volume alone differs simply
because the three corpora start from different original depths and
different mass-effect sensitivities.

\paragraph{Status.} We report this analysis as exploratory: inference-only masking on the
existing trained checkpoints (three seeds pooled per corpus), with no
retraining. We do not treat it as resolving the
origin of the corpus-dependent depth of compression (that remains
open); it is a complementary,
within-corpus probe suggesting that natural-language and code corpora
differ not in \emph{which} tokens dominate cross-paragraph attention,
nor in \emph{how much} additional effect masking them has beyond
volume alone, but in how much of the original well a pure volume
effect already accounts for --- which determines whether the same,
roughly constant hub-specific increment leaves a corpus's well
partially eroded, fully erased, or driven into reversal.

\section{Vocabulary Robustness of the Corpus Structural Scale}
\label{app:vocab-robust}

Paragraph distributions depend on the vocabulary representing them, so we
test whether $\lstruct$'s cross-corpus ordering is an artifact of one
cutoff, recomputing it at $V\in\{50,100,150,200,300,500\}$
(Table~\ref{tab:vocab}). The main-text values instead use each corpus's
full frequency-filtered vocabulary ($V=143$, $324$, and $163$ for Code,
WikiText-2, and OpenWebText), which is why they coincide with no row of
the table.

\begin{table*}[!htbp]
\centering
\caption{
Vocabulary robustness of the corpus structural scale: $\lstruct$
recomputed at each vocabulary cutoff $V$ in a separate sweep. The table
checks only whether the cross-corpus ordering of $\lstruct$ changes with
$V$; its entries come from a different fitting run than the main-text
bootstrap medians (Table~\ref{tab:corpus}) and are not comparable to them
in absolute value. For Code and OpenWebText, the effective vocabulary
reaches the available frequency-filtered vocabulary ceiling at larger $V$.
}
\label{tab:vocab}
\begin{tabular}{c|ccc}
\toprule
$V$ & Code & WikiText-2 & OpenWebText\\
\midrule
50
& 4.950
& 4.347
& 4.555
\\
100
& 4.667
& 4.147
& 4.652
\\
150
& 4.991
& 4.048
& 4.165
\\
200
& 4.991
& 3.901
& 4.155
\\
300
& 4.991
& 3.774
& 4.155
\\
500
& 4.991
& 3.709
& 4.155
\\
\bottomrule
\end{tabular}
\end{table*}

Although WikiText-2's $\lstruct$ decreases moderately as the vocabulary
expands, $\lstruct$'s cross-corpus ordering is unchanged at every
tested size (Code $>$ OpenWebText $>$ WikiText-2). So
$\lstruct$ is stable and well-defined, not an artifact of one cutoff,
though its absolute value is not fully vocabulary-independent. This
does not rescue the correspondence with $U^*$: as
Section~\ref{sec:layer3-main} established, $\lstruct$'s ordering
matches compression depth at no vocabulary size, since the mismatch is
between $\lstruct$ and $U^*$, not within $\lstruct$ across
vocabularies.

\end{document}